\documentclass{article}
\usepackage{iclr2027_conference,times}
\usepackage[T1]{fontenc}
\usepackage[utf8]{inputenc}
\usepackage{amsmath,amssymb,amsthm,mathtools}
\usepackage{graphicx,booktabs,tabularx,array,multirow}
\usepackage{algorithm}
\usepackage{algpseudocode}
\usepackage{xcolor}
\usepackage{hyperref}
\usepackage{xurl}
\hypersetup{hidelinks,pdftitle={RA-MoWE: Workflow-Affinity Embeddings for Query Clustering and Agentic Workflow Generation},pdfauthor={Anonymous}}
\newtheorem{proposition}{Proposition}
\newtheorem{theorem}{Theorem}
\newtheorem{lemma}{Lemma}

\theoremstyle{remark}

\newcommand{\E}{\mathbb{E}}

\newcommand{\method}{RA-MoWE}
\DeclareMathOperator*{\argmax}{arg\,max}
\DeclareMathOperator*{\argmin}{arg\,min}

\title{RA-MoWE: Workflow-Affinity Embeddings\\for Query Clustering and\\Agentic Workflow Generation}
\author{
  Qi Cheng \\
  Rutgers University\\
  New Brunswick, NJ\\
  \texttt{q.cheng@rutgers.edu} \\
  \And
  Shengyu Chen \\
  University of Pittsburgh \\
  Pittsburgh, PA \\
  \texttt{shc160@pitt.edu} \\
  \And
  Wei Cheng \\
  NEC Labs America\\
  Princeton, NJ \\
  \texttt{weicheng@nec-labs.com} \\
  \And
  Yiqun Xie \\
  University of Maryland\\
  College Park, MD \\
  \texttt{xie@umd.edu} \\
  \And
  Haoyu Wang \\
  NEC Labs America\\
  Princeton, NJ \\
  \texttt{haoyu@nec-labs.com} \\
  \And
  Haifeng Chen \\
  NEC Labs America\\
  Princeton, NJ \\
  \texttt{haifeng@nec-labs.com} \\
  \and
  Xiaowei Jia \\
  Rutgers University\\
  New Brunswick, NJ\\
  \texttt{xj159@cs.rutgers.edu} \\
}
\iclrfinalcopy
\begin{document}
\maketitle
\begin{abstract}
Agentic workflows enable large language models (LLMs) to solve complex tasks by coordinating reasoning, tool use, and verification. However, a workflow optimized for an entire task collection can overlook differences in the reasoning strategies that individual queries need, while searching for a new workflow for every query repeats costly optimization. To address this tradeoff, we introduce RA-MoWE, a framework that uses workflow-affinity embeddings to cluster queries and guide the generation of reusable expert workflows. Each embedding records how well a fixed set of reference workflows solves a query, revealing similarities in which reasoning strategies are effective. RA-MoWE uses each cluster's queries and average embedding to initialize and refine a specialized workflow through execution feedback. An embedding encoder predicts these embeddings from query text, allowing new queries to select a generated expert without first executing the reference workflows. On a 300-query test set drawn from four benchmarks spanning mathematics, science, and programming, RA-MoWE improves average task score by 4.04 percentage points over selecting among the reference workflows, while using 27.7\% fewer language-model calls at inference.
\end{abstract}
\section{Introduction}
\label{sec:introduction}
Agentic workflows organize large language model (LLM) calls into procedures for reasoning, revising answers, and using tools, extending their capabilities on complex tasks \citep{react2023,selfrefine2023}. Their design determines which operations run and how they interact, affecting both solution quality and inference cost: repeated sampling, revision, and decomposition offer different ways to spend computation. Manually choosing a suitable design for a diverse collection of queries is therefore a substantial engineering effort \citep{adas2025,aflow2025}.

Recent methods automate workflow design through search and learned generation. ADAS and AFlow explore code-represented agents using execution feedback, while AgentSquare searches combinations of reusable modules \citep{adas2025,aflow2025,agentsquare2025}. 
In these works, the workflow is often optimized over a collection of tasks using a collection-level score 
and thus can overlook differences in individual query need. 
In contrast, some approaches, such as FlowReasoner~\citep{flowreasoner2025}, learn to generate a workflow for each query. These methods need to repeat costly optimization and also have less training samples for each workflow. 
Addressing such tradeoff requires 
identifying which queries should share a workflow and what computations each group needs. 
While a natural solution is to group queries based on their similarity. However, conventional semantic similarity does not necessarily capture whether queries benefit from the same reasoning strategies.

Our key idea is to introduce a \emph{workflow-affinity embedding} and \emph{group queries by how they respond to different reasoning workflows}. 
In particular, the \emph{workflow-affinity embedding} is a vector whose coordinates estimate a query's task score under a fixed set of reference workflows, called probes. The resulting vector captures a query's relative strengths and weaknesses across different reasoning strategies, including which reasoning strategies succeed, partially succeed, or fail. Queries with similar profiles can therefore be grouped together based on their computational needs rather than their semantic similarity.


We integrate this idea into \method{}, a framework that uses the workflow-affinity embeddings to guide the generation of reusable expert workflows. Offline, we execute the reference workflows on training queries, compute their embeddings, and partition them into clusters. Clustering these embeddings groups queries with similar measured responses, giving expert construction both a focused set of examples and evidence about potentially useful reasoning strategies.   For each cluster, we introduce an embedding-guided search mechanism, which uses the cluster prototype to choose an executable starting workflow and constrain the candidate operations. It then proposes refinements informed by execution feedback and probe outcomes on remaining failures, retaining edits that improve the cluster's validation score. The resulting expert is stored for reuse. Figure~\ref{fig:affinity-profiles} connects the measured cluster affinities to the workflows produced by this process.

\begin{figure}[t]
\centering
\includegraphics[width=\linewidth]{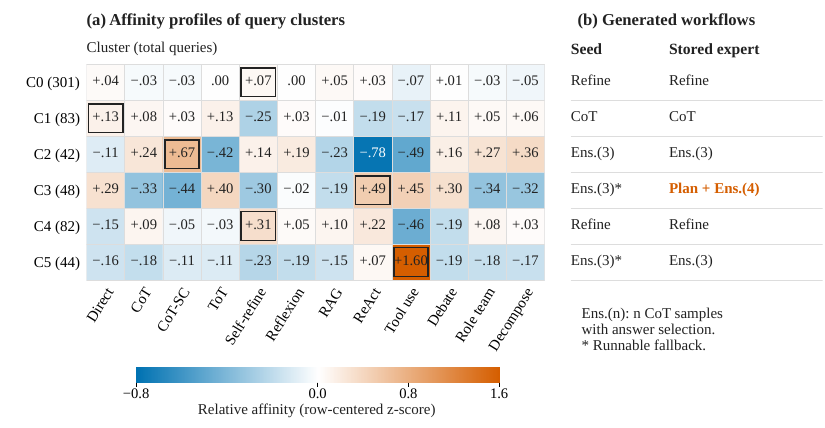}
\caption{\textbf{Different workflow preferences motivate specialized experts.} (a) Average embeddings for six GPT-4o-mini training-query clusters show distinct responses to reference workflows. C2 favors sampling and answer selection (CoT-SC), whereas C4 favors self-refinement. (b) \method{} uses these preferences to choose initial workflows. Values are standardized scores centered within each row; outlines mark row maxima, and parentheses give full cluster sizes. Refine revises a chain-of-thought (CoT) answer; Ens.($n$) selects among $n$ sampled CoT answers; Plan adds decomposition. Asterisks mark runnable substitutes for unavailable workflows. Measurement details are in Appendix~\ref{app:guided-construction}.}
\label{fig:affinity-profiles}
\end{figure}

We then create an \emph{embedding encoder}, which  makes the generated experts accessible from query text alone. 
While workflow-affinity embeddings can be measured offline by executing the reference workflows, this process is too costly to repeat for every new query. We therefore train the encoder to predict a query's workflow-affinity embedding directly from its text, using the measured embeddings as supervision. 
At deployment, the predicted embedding selects a cluster's expert without executing the reference workflows.

We evaluate \method{} across mathematics, science, programming, and repository repair against six families of adapted search and routing methods. On 300 test queries drawn from four benchmarks in mathematics, science, and programming, \method{} improves average task score by 4.04 percentage points over a system that selects from the reference workflows, while using 27.7\% fewer LLM calls at inference. Both systems choose a workflow based on the query text. Further experiments examine whether queries consistently benefit from particular workflows, how clustering affects the quality of generated workflows, and how accurately the encoder assigns new queries to experts.

Our contributions are:
\begin{enumerate}
\setlength{\itemsep}{0pt}
\item \textbf{Workflow-affinity embeddings} that represent queries through measured workflow outcomes and organize training examples for specialist construction.
\item \textbf{An embedding encoder} that predicts this representation from text, enabling expert selection without online reference-workflow execution.
\item \textbf{Embedding-guided expert generation} that uses cluster affinities, query examples, and failure feedback to initialize and refine reusable workflows.
\end{enumerate}
Our analysis identifies when measurement and prediction errors preserve useful decisions and when the available workflow operations support transfer from a strong probe to a generated expert.

\section{RA-MoWE: embedding-guided workflow generation}
\label{sec:problem}
Given training queries with reference answers or executable tests, our goal is to construct reusable workflows that serve different computational needs. \method{} uses a shared representation to connect three decisions: which queries should be grouped together, how to build an expert for each group, and which expert should answer a new query. We first measure each query's response to reference workflows and cluster the resulting embeddings. These embeddings guide expert generation and supervise an encoder that selects an expert from query text. Figure~\ref{fig:pipeline} illustrates the framework; Algorithm~\ref{alg:ramowe} summarizes its construction and deployment.

\begin{figure}[!htbp]
\centering
\includegraphics[width=\linewidth]{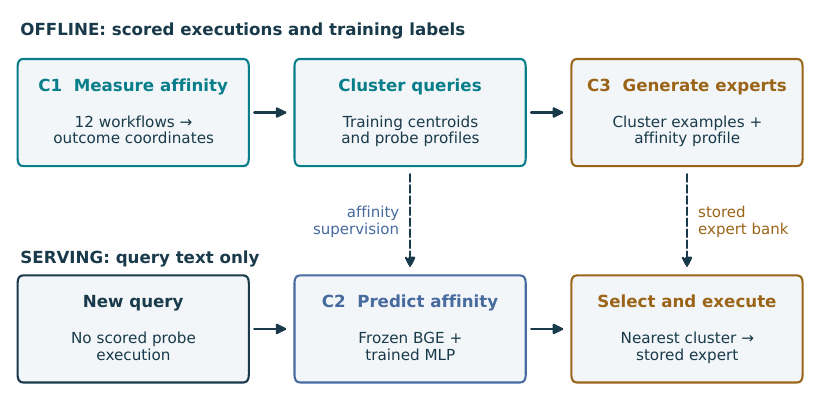}
\caption{The RA-MoWE framework. Offline, reference-workflow scores define query embeddings for clustering, expert generation, and encoder supervision. At deployment, the encoder predicts an embedding from query text and selects a stored expert. Dashed arrows indicate saved supervision and experts. C1--C3 mark the embedding, encoder, and generation contributions.}
\label{fig:pipeline}
\end{figure}

\subsection{Workflow-affinity embeddings for query clustering}
\label{sec:embedding}
A useful query group should reflect which reasoning procedures help its members. Questions about the same topic can require different computations, such as comparing several candidate answers or revising an initial solution. We therefore represent a query by its performance under a fixed panel of diverse reference workflows, called \emph{probes}. These probes apply established strategies such as chain-of-thought, self-consistency, and self-refinement \citep{cot2022,selfconsistency2023,selfrefine2023}.

For a query $q$ and probe $p_k$, let $Y_{k,s}(q)\in[0,1]$ be the task score from execution $s$ under a fixed execution and scoring protocol. Averaging $S$ repetitions gives
\begin{equation}
\bar f_k(q)=\frac{1}{S}\sum_{s=1}^{S}Y_{k,s}(q),\qquad
\bar f(q)=\bigl(\bar f_1(q),\ldots,\bar f_K(q)\bigr)^\top,
\label{eq:affinity}
\end{equation}
where $K$ is the number of probes. The vector $\bar f(q)$ retains the strengths and weaknesses of all reference workflows, including partial successes and multiple successful strategies. We standardize its coordinates using training statistics to obtain the \emph{workflow-affinity embedding} $\bar z(q)$. This prevents differences in score variance from dominating the comparison between queries; the normalization and measurement settings are given in Appendix~\ref{app:method}.

We cluster the training queries in the  embedding space using $k$-means, yielding $J$ groups with centers $C=(c_1,\ldots,c_J)$. Each group supplies examples for a reusable expert. 
Grouping queries by their workflow responses therefore provides workflow construction with both relevant examples and evidence about which reasoning strategies are effective for those examples. 
Such evidence also has a precise interpretation within the probe panel: if two queries' expected probe scores differ by at most $\epsilon$ in every coordinate, transferring the best probe from one to the other loses at most $2\epsilon$ in expected score. Proposition~\ref{prop:measurement} and its proof relate this decision guarantee to measurement accuracy. The next step uses the resulting groups to construct workflows beyond the reference panel.

\subsection{Embedding-guided expert generation}
\label{sec:generation}
Each cluster supplies two complementary inputs for expert generation: (1) examples to optimize and (2) evidence about useful reasoning operations. For cluster $j$, let $V_j$ be the subset of its training queries used to evaluate candidate workflows. We summarize their measured embeddings by
\begin{equation}
\rho_j=\frac{1}{|V_j|}\sum_{q\in V_j}\bar z(q),
\label{eq:cluster-profile}
\end{equation}
where $|V_j|$ is the number of queries in this subset. The prototype vector $\rho_j$ describes the cluster's relative strengths across the reference workflows. It provides search with a starting strategy and evidence for choosing subsequent operations.

Search begins by mapping the highest-ranked runnable probe in $\rho_j$ to an initial workflow. We construct workflows from executable operations such as planning, sampling, answer selection, and revision, following the modular-search principle of \citet{aflow2025,agentsquare2025}. The same embedding restricts candidate operations using the lowest-ranked probes, while retaining the operation chosen for initialization. In this way, the measured response guides both where search starts and which changes it considers. The value of the initialization depends on how well the executable workflow preserves the probe's useful behavior; Proposition~\ref{prop:seed} characterizes this connection.

Refinement focuses on the queries the current workflow has not solved. An LLM optimizer receives failure examples, summaries of probe scores on failed and solved queries, and previous search attempts. These summaries identify strategies that may address the remaining failures. The optimizer proposes an edit, which is executed on $V_j$ and retained only if its average score improves by a prescribed margin. Repeating this process produces expert $w_j$; the collection $\mathcal W=\{w_1,\ldots,w_J\}$ forms the reusable expert bank. The operator mapping, proposal rules, and acceptance margin are specified in Appendix~\ref{app:method}. Search uses construction data, while the separate test pool evaluates the resulting bank.

\begin{algorithm}[!t]
\caption{RA-MoWE: constructing and using a bank of workflow experts}
\label{alg:ramowe}
\begin{algorithmic}[1]
\Require Training queries and task scorers; reference workflows; cluster count $J$; search resources
\State Measure and standardize the workflow-affinity embedding $\bar z(q)$ of each training query
\State Train the embedding encoder to predict $\bar z(q)$ from query text
\State Cluster the measured embeddings into $J$ groups with centers $C$
\For{each cluster $j$}
  \State Choose construction queries $V_j$ and compute their average embedding $\rho_j$
  \State Initialize $w_j$ from a high-affinity runnable probe; restrict candidate operations using $\rho_j$
  \For{each allowed search round}
    \State Use probe scores and remaining failures to guide an LLM-proposed edit
    \State Evaluate on $V_j$; accept only if the score improvement exceeds the required margin
  \EndFor
  \State Store the final expert $w_j$
\EndFor
\State \Return Trained encoder, cluster centers $C$, and expert bank $\mathcal W$
\Statex \textbf{Deployment:} predict a new query's embedding, select the nearest cluster, and run its expert.
\end{algorithmic}
\end{algorithm}

Algorithm~\ref{alg:ramowe} presents the shared-clustering framework; the recorded alignment between construction and deployment clusters is detailed in Appendix~\ref{app:guided-construction}.

\subsection{Learning the embedding encoder}
\label{sec:encoder}
The embeddings used for construction require scored probe executions. To select an expert for a new query without repeating that measurement, we train an encoder to predict the same coordinates from text. Let $\psi(q)$ be frozen text features and $g_\theta$ a trainable prediction head with parameters $\theta$. The predicted embedding is
\begin{equation}
\hat z(q)=g_\theta(\psi(q))\in\mathbb R^K.
\label{eq:encoder-prediction}
\end{equation}
Training pairs each query with its measured embedding $\bar z(q)$. We use frozen BGE features \citep{bge2024} as inputs to 
a small prediction head, which is trained to predict the measured embedding from query text.  
The loss encourages agreement in the relative pattern of probe responses and penalizes errors in the coordinates used for clustering. Predicting the full embedding also allows the same learned target to support different cluster partitions. The loss, architecture, and training protocol are specified in Appendix~\ref{app:method}.

At deployment, the predicted embedding selects the closest cluster center:
\begin{equation}
a(q)=\argmin_{1\le j\le J}\|\hat z(q)-c_j\|_2^2.
\label{eq:deployment}
\end{equation}
Here $c_j$ is the center of cluster $j$, $\|\cdot\|_2$ is the Euclidean norm, and $a(q)$ is the selected expert's index. We then execute the stored workflow $w_{a(q)}$ on $q$. Both probe measurement and workflow construction are therefore performed offline and reused across future queries.

The prediction need not recover every coordinate exactly to preserve an expert choice. If the prediction error is smaller than the measured embedding's distance to the nearest cluster boundary, the two embeddings select the same expert. Proposition~\ref{prop:routing} and Appendix~\ref{app:routing-proof} formalize this condition and its counterpart for expected probe scores. This motivates evaluating the encoder through both prediction error and the resulting expert assignments. Appendix~\ref{app:composed-bound} combines the conditions for measurement, generation, and assignment into an analysis of the complete pipeline.

\section{Empirical evaluation}
\label{sec:experiments}

We test the three links in RA-MoWE: whether workflow responses reveal reusable specialization, whether the encoder preserves useful assignments from text, and whether affinity profiles guide the construction of effective workflows. Each experiment connects an intermediate design choice to downstream utility.

\paragraph{Tasks and evaluation.}
MixBench-H contains 600 training and 300 test queries spanning AIME, GPQA, CodeContests, and LiveCodeBench~\citep{aime_source,gpqa2024,codecontests2022,livecodebench2024}. We evaluate GPT-4o-mini and Claude Haiku 4.5 separately. SWE-bench Verified~\citep{swebench2024} uses a repository-stratified 300/200 split and a fixed Agentless-lite localize--repair pipeline~\citep{agentless2025}. Mixed-task utility is mean answer correctness or public-test pass fraction; SWE utility is harness resolution. Comparisons within each regime share queries and scoring contracts. The main comparison tests the embedding-to-bank interface using cluster-conditioned AFlow search on mixed tasks and cluster policy selection on SWE. A separate experiment evaluates embedding-guided expert generation. We report paired stratified bootstrap intervals and Holm-adjusted sign-flip tests for the historical routing families; the embedding-guided search contrasts use exploratory, unadjusted query-bootstrap intervals.

\begin{table}[t]
\centering\small
\setlength{\tabcolsep}{5pt}
\caption{Baseline comparison on three regimes (higher utility is better). Scores are comparable within columns. $A^{\dagger}$ uses measured test affinities; $B$ predicts them from text. $\dagger$ indicates scorer feedback on the evaluated query, and $\ddagger$ public-test execution feedback. FlowReasoner has $\dagger$ only on mixed tasks. Search baselines are the documented local adaptations, not published full-system reproductions. Dashes denote unrun configurations. Bold identifies significant gains over $D$ under measured assignment.}
\label{tab:historical}
\label{tab:main-baselines}
\begin{tabular}{lrrr}
\toprule
 & \multicolumn{2}{c}{MixBench-H} & SWE Verified\\
Method & GPT-4o-mini & Haiku 4.5 & GPT-4o-mini\\
\midrule
Direct & .2872 & .5694 & .2000\\
Train-selected probe & $.3481^{\ddagger}$ & .6894 & .1750\\
Global comparator $D$ & .2722 & .5381 & .1750\\
AFlow search & .2722 & .5381 & .1900\\
ADAS (executor-class meta-model) & .2042 & .2489 & .2000\\
ScoreFlow (SF-0/ICPO) & .2756 & --- & .2100\\
FlowReasoner (search only) & $.3565^{\dagger}$ & $.7465^{\dagger}$ & .2050\\
MasRouter (policy routing) & --- & --- & .1700\\
FrugalGPT (policy cascade) & --- & --- & .2050\\
\midrule
Semantic BGE clusters & .2939 & .5303 & .1800\\
Semantic Qwen3 clusters & .3050 & .5861 & ---\\
Random clusters & .2767 & .5228 & .1800\\
RA-MoWE: text-predicted $B$ & .2922 & .5214 & .2000\\
RA-MoWE: measured $A^{\dagger}$ & .2742 & \textbf{.6097} & \textbf{.2450}\\
\bottomrule
\end{tabular}
\end{table}

\paragraph{Comparisons test grouping and construction choices.}
Table~\ref{tab:main-baselines} compares direct execution, a training-selected probe, and adaptations of AFlow~\citep{aflow2025}, ADAS~\citep{adas2025}, ScoreFlow~\citep{scoreflow2025}, and FlowReasoner~\citep{flowreasoner2025}. Semantic BGE/Qwen3~\citep{bge2024,qwen3embedding2025} and random partitions test the choice of grouping signal; SWE also includes policy-level MasRouter~\citep{masrouter2025} and FrugalGPT~\citep{frugalgpt2024}. The global comparator $D$ searches the undivided mixed-task pool with the same search rule and five optimization rounds after initialization. Both settings allocate 600 validation queries per round in total, although the six-cluster bank requires more proposal calls. SWE $D$ is the training-selected policy. Appendix~\ref{app:baseline-adaptations} details the adaptations and resource accounting; Appendix~\ref{app:router-cost} reports heterogeneous-model routing separately.

\subsection{Workflow affinity captures reusable specialization}

\paragraph{Probe preferences persist across executions.}
The first test asks whether affinity captures a repeatable response signal. We select a probe using two executions per test query and evaluate it on the third, averaging all three folds. GPT-4o-mini reaches $0.4059$ against $0.3481$ for the training-selected fixed probe: $+0.0579$ (95\% CI $[0.0283,0.0881]$). Uniform treatment of tied probes retains a $0.0382$ gain ($[0.0100,0.0668]$). The improvement survives an execution not used for selection, supporting query-specific workflow preference within the panel. Haiku reaches $0.7237$ against $0.6894$, but uniform ties reduce its gain to $0.0015$; repeatability therefore depends on the backbone and tie rule. Table~\ref{tab:repeatability} gives both rules. Increasing the panel from eight to twelve probes also raises distinct training signatures from 281 to 329 and effective rank from 2.67 to 3.08, giving clustering a richer description of query behavior.

\paragraph{Affinity-defined specialists improve the global comparator.}
With measured test-affinity assignment $A$, Haiku utility rises from $D$'s $0.5381$ to $0.6097$: $+0.0717$ ($[0.0339,0.1122]$, Holm $p=0.0063$). On SWE, the bank resolves 49/200 instances against 35/200 for $D$: $+0.0700$ ($[0.0300,0.1100]$, $p=0.0052$), correcting 16 failures while introducing two. Measured-affinity assignment also exceeds semantic BGE and random partitions by $0.0650$ on SWE (Holm $p=0.0164/0.0123$). Its $24.5\%$ resolution rate is the highest point estimate among the tested SWE adaptations, including ScoreFlow at $21.0\%$ and FlowReasoner/FrugalGPT at $20.5\%$. These results show that workflow-response geometry provides an actionable basis for specialist assignment. Because $A$ uses test outcomes, the next experiment measures how much of this benefit transfers to text-only decisions.

The benefit varies across regimes: GPT-4o-mini's $A-D$ gap is $0.0019$, with no historical routing contrast surviving Holm correction. Outcome-assisted FlowReasoner leads the mixed-task columns, and Haiku's fixed self-consistency probe reaches $0.6894$. Thus the evidence supports useful specialization relative to global construction in two regimes, while retaining strong alternatives in the comparison.

\subsection{Text prediction connects the embedding to deployment}

The encoder must preserve decisions that matter to the expert bank. On the expanded panel, the frozen-BGE MLP obtains held-out cosine similarity $0.398$ and coordinate-averaged MSE $0.870$ with nested training on a 510/90 split. It replaces probe executions with one text forward pass. Downstream, $B$ reaches $0.2922/0.5214/0.2000$ in the three regimes of Table~\ref{tab:main-baselines}; none of its prespecified contrasts against $D$ is significant.

Assignment diagnostics explain why reconstruction metrics alone are insufficient. Measured/predicted cluster agreement is 37.7\% for GPT-4o-mini and 12.0\% for Haiku, whose predictor assigns 277/300 queries to one cluster. The SWE predictor chooses direct for all 200 instances. The measured-affinity gains therefore establish a useful target, but the current predictors do not consistently preserve its specialist decisions. This connects the empirical challenge to Proposition~\ref{prop:routing}, which assesses prediction error relative to cluster boundaries. Assignment fidelity and expert utility are therefore essential complements to coordinate error.

\subsection{Embedding-guided search generates effective workflows}

\begin{figure}[htbp]
\centering
\includegraphics[width=\linewidth]{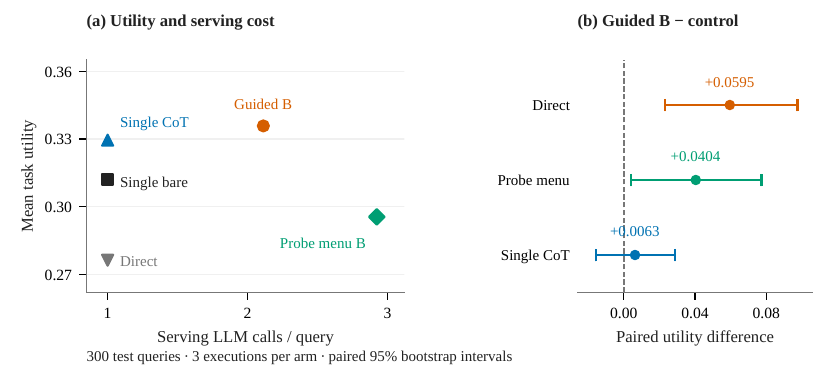}
\caption{Embedding-guided expert generation on 300 MixBench-H queries with three executions per arm. Left: utility versus serving calls/query. Right: Guided-$B$ utility differences with exploratory 95\% paired query-bootstrap intervals, resampling per-query execution means. Guided-$B$ improves on Direct and routed probes; its difference from single-call CoT is unresolved. Serving calls exclude offline measurement and construction. Prompting, decoding, and operator differences make the contrasts comparisons of complete configurations.}
\label{fig:controlled-results}
\end{figure}

We next evaluate embedding-guided search, which uses cluster preferences to initialize and edit executable workflows. This study uses GPT-4o-mini through the official OpenAI endpoint, separately from the preceding NEC-endpoint experiments. Each of six cluster searches evaluates an affinity-derived seed and at most two edits. The resulting bank includes refinement, sampling/selection, and decomposition. Five experts retain their seeds; cluster C3 accepts a larger ensemble followed by decomposition. The search traces show that behavioral initialization supplies most of the final bank, with local edits refining one specialist.

\paragraph{Generated workflows improve on Direct and routed probes.}
Figure~\ref{fig:controlled-results} compares the generated bank with fixed and routed workflows. On an independent 300-query test pool, Guided-$B$ scores $0.3358$ against Direct's $0.2763$: $+0.0595$ ($[0.0232,0.0975]$, exploratory $p=0.001$). It also improves over probe-menu-$B$'s $0.2955$ by $0.0404$ ($[0.0042,0.0773]$, $p=0.030$), while reducing calls/query from 2.924 to 2.113 (27.7\%). These comparisons show that the generated bank provides effective, reusable computation in a text-only system. Prompting, decoding, and execution differ across these arms, so the improvement is attributed to the complete configuration.

The same-executor single-call CoT control reaches $0.3295$; its $0.0063$ gap from Guided-$B$ remains unresolved ($[-0.0158,0.0289]$). A separate six-expert evaluation gives an observed per-query oracle of $0.4961$ versus $0.3567$ for the test-selected best fixed expert, while current routing realizes approximately $0.34$. The remaining opportunity is concrete: the bank contains complementary observed successes, and better text prediction is needed to recover them reliably. Appendix~\ref{app:experiments} reports expert semantics, costs, panel and operator ablations, label-free features, and transfer across regimes.

\section{Related work}
\label{sec:related}
\paragraph{Learning representations for computational decisions.}
Instance-dependent algorithm selection connects instance features to algorithm performance \citep{rice1976}; decision-relevant representation learning seeks to preserve distinctions that affect action quality \citep{zhang2021bisimulation}. Workflow-affinity embeddings specialize this perspective to query-level outcomes of executable reasoning procedures. They provide an external behavioral representation, whereas function vectors and task arithmetic represent tasks through internal activations and weight differences \citep{todd2024function,ilharco2023task}. General text embeddings remain useful features \citep{mteb2023,bge2024,qwen3embedding2025}; our supervision target additionally names the workflow behaviors that clustering and construction should exploit.

\paragraph{Allocating workflow construction effort.}
Search-based agent design learns executable programs or modular compositions from performance feedback \citep{adas2025,aflow2025,agentsquare2025}. Learned workflow generators further amortize design through preference optimization or execution-reward training \citep{scoreflow2025,flowreasoner2025}. RA-MoWE organizes construction around a reusable behavioral group. Its cluster determines the examples on which an expert is optimized, while the profile supplies a seed preference and failure-conditioned guidance. This interface accommodates different search backends and makes the grouping representation part of expert construction itself.

\paragraph{Selecting models and orchestrating agents.}
Preference-trained routing, cost-aware cascades, and joint selection of roles and models adapt computation to a query \citep{routellm2025,frugalgpt2024,masrouter2025}. Related approaches align internal expert routing with task neighborhoods, learn query-conditioned agentic architectures, or train coordination policies \citep{roma2026,maas2025,conductor2026}. RA-MoWE connects pre-execution selection to the creation of an external workflow bank through one outcome-coordinate space. Our experiments compare the implemented local adaptations under their disclosed information and resource contracts; the separate model-pool study examines the additional freedom to change backbones.

\section{Discussion and conclusion}
\label{sec:discussion}
Workflow affinity provides a common basis for learning query representations and constructing the workflows that serve them. The experiments show repeatable probe complementarity, gains from measured-affinity specialization in two regimes, and effective compact workflows initialized from cluster profiles. The encoder makes this construction reusable from query text; its assignment fidelity determines how much of the behavioral structure reaches deployment.

\label{maintext:end}
\bibliography{references}

@inproceedings{aflow2025,
  title={{AFlow}: Automating Agentic Workflow Generation},
  author={Zhang, Jiayi and Xiang, Jinyu and Yu, Zhaoyang and Teng, Fengwei and Chen, Xiong-Hui and Chen, Jiaqi and Zhuge, Mingchen and Cheng, Xin and Hong, Sirui and Wang, Jinlin and Zheng, Bingnan and Liu, Bang and Luo, Yuyu and Wu, Chenglin},
  booktitle={International Conference on Learning Representations}, year={2025},
  url={https://proceedings.iclr.cc/paper_files/paper/2025/hash/5492ecbce4439401798dcd2c90be94cd-Abstract-Conference.html}
}

@inproceedings{adas2025,
  title={Automated Design of Agentic Systems},
  author={Hu, Shengran and Lu, Cong and Clune, Jeff},
  booktitle={International Conference on Learning Representations}, year={2025},
  url={https://proceedings.iclr.cc/paper_files/paper/2025/hash/36b7acf6f6010652b3f2a433774a66fe-Abstract-Conference.html}
}

@inproceedings{agentsquare2025,
  title={{AgentSquare}: Automatic {LLM} Agent Search in Modular Design Space},
  author={Shang, Yu and Li, Yu and Zhao, Keyu and Ma, Likai and Liu, Jiahe and Xu, Fengli and Li, Yong},
  booktitle={International Conference on Learning Representations}, year={2025},
  url={https://proceedings.iclr.cc/paper_files/paper/2025/hash/0ae94013da7cd459402fd77874e09ee3-Abstract-Conference.html}
}

@inproceedings{todd2024function,
  title={Function Vectors in Large Language Models},
  author={Todd, Eric and Li, Millicent L. and Sharma, Arnab Sen and Mueller, Aaron and Wallace, Byron C. and Bau, David},
  booktitle={International Conference on Learning Representations}, year={2024},
  url={https://proceedings.iclr.cc/paper_files/paper/2024/hash/4ae163cb8788970e53b4fd9578141139-Abstract-Conference.html}
}

@inproceedings{ilharco2023task,
  title={Editing Models with Task Arithmetic},
  author={Ilharco, Gabriel and Ribeiro, Marco Tulio and Wortsman, Mitchell and Gururangan, Suchin and Schmidt, Ludwig and Hajishirzi, Hannaneh and Farhadi, Ali},
  booktitle={International Conference on Learning Representations}, year={2023},
  url={https://openreview.net/forum?id=6t0Kwf8-jrj}
}

@inproceedings{zhang2021bisimulation,
  title={Learning Invariant Representations for Reinforcement Learning without Reconstruction},
  author={Zhang, Amy and McAllister, Rowan and Calandra, Roberto and Gal, Yarin and Levine, Sergey},
  booktitle={International Conference on Learning Representations}, year={2021},
  url={https://openreview.net/forum?id=-2FCwDKRREu}
}

@inproceedings{routellm2025,
  title={{RouteLLM}: Learning to Route {LLM}s with Preference Data},
  author={Ong, Isaac and Almahairi, Amjad and Wu, Vincent and Chiang, Wei-Lin and Wu, Tianhao and Gonzalez, Joseph E. and Kadous, M. and Stoica, Ion},
  booktitle={International Conference on Learning Representations}, year={2025},
  url={https://proceedings.iclr.cc/paper_files/paper/2025/hash/5503a7c69d48a2f86fc00b3dc09de686-Abstract-Conference.html}
}

@inproceedings{roma2026,
  title={Routing Manifold Alignment Improves Generalization of Mixture-of-Experts {LLM}s},
  author={Li, Zhongyang and Li, Ziyue and Zhou, Tianyi},
  booktitle={International Conference on Learning Representations}, year={2026},
  url={https://proceedings.iclr.cc/paper_files/paper/2026/hash/249b8d8f41970822651435629e68a6e1-Abstract-Conference.html}
}

@inproceedings{maas2025,
  title={Multi-agent Architecture Search via Agentic Supernet},
  author={Zhang, Guibin and Niu, Luyang and Fang, Junfeng and Wang, Kun and Bai, Lei and Wang, Xiang},
  booktitle={Proceedings of the 42nd International Conference on Machine Learning},
  series={Proceedings of Machine Learning Research}, volume={267}, pages={75834--75852}, year={2025},
  url={https://proceedings.mlr.press/v267/zhang25bi.html}
}

@inproceedings{conductor2026,
  title={Learning to Orchestrate Agents in Natural Language with the Conductor},
  author={Nielsen, Stefan and Cetin, Edoardo and Schwendeman, Peter and Sun, Qi and Xu, Jinglue and Tang, Yujin},
  booktitle={International Conference on Learning Representations}, year={2026},
  url={https://openreview.net/forum?id=U23A2BUKYt}
}

@article{hoeffding1963,
  title={Probability Inequalities for Sums of Bounded Random Variables},
  author={Hoeffding, Wassily}, journal={Journal of the American Statistical Association},
  volume={58}, number={301}, pages={13--30}, year={1963},
  doi={10.1080/01621459.1963.10500830}
}

@incollection{rice1976,
  title={The Algorithm Selection Problem}, author={Rice, John R.},
  booktitle={Advances in Computers}, volume={15}, pages={65--118}, publisher={Elsevier}, year={1976},
  doi={10.1016/S0065-2458(08)60520-3}
}

@inproceedings{cot2022,
  title={Chain-of-Thought Prompting Elicits Reasoning in Large Language Models},
  author={Wei, Jason and Wang, Xuezhi and Schuurmans, Dale and Bosma, Maarten and Ichter, Brian and Xia, Fei and Chi, Ed H. and Le, Quoc V. and Zhou, Denny},
  booktitle={Advances in Neural Information Processing Systems}, volume={35}, year={2022},
  url={https://proceedings.neurips.cc/paper/2022/hash/9d5609613524ecf4f15af0f7b31abca4-Abstract-Conference.html}
}

@inproceedings{selfconsistency2023,
  title={Self-Consistency Improves Chain of Thought Reasoning in Language Models},
  author={Wang, Xuezhi and Wei, Jason and Schuurmans, Dale and Le, Quoc V. and Chi, Ed H. and Narang, Sharan and Chowdhery, Aakanksha and Zhou, Denny},
  booktitle={International Conference on Learning Representations}, year={2023},
  url={https://openreview.net/forum?id=1PL1NIMMrw}
}

@inproceedings{react2023,
  title={{ReAct}: Synergizing Reasoning and Acting in Language Models},
  author={Yao, Shunyu and Zhao, Jeffrey and Yu, Dian and Du, Nan and Shafran, Izhak and Narasimhan, Karthik and Cao, Yuan},
  booktitle={International Conference on Learning Representations}, year={2023},
  url={https://openreview.net/forum?id=WE_vluYUL-X}
}

@inproceedings{selfrefine2023,
  title={{Self-Refine}: Iterative Refinement with Self-Feedback},
  author={Madaan, Aman and Tandon, Niket and Gupta, Prakhar and Hallinan, Skyler and Gao, Luyu and Wiegreffe, Sarah and Alon, Uri and Dziri, Nouha and Prabhumoye, Shrimai and Yang, Yiming and Gupta, Shashank and Majumder, Bodhisattwa Prasad and Hermann, Katherine and Welleck, Sean and Yazdanbakhsh, Amir and Clark, Peter},
  booktitle={Advances in Neural Information Processing Systems}, volume={36}, year={2023},
  url={https://proceedings.neurips.cc/paper_files/paper/2023/hash/91edff07232fb1b55a505a9e9f6c0ff3-Abstract-Conference.html}
}

@inproceedings{tot2023,
  title={Tree of Thoughts: Deliberate Problem Solving with Large Language Models},
  author={Yao, Shunyu and Yu, Dian and Zhao, Jeffrey and Shafran, Izhak and Griffiths, Thomas L. and Cao, Yuan and Narasimhan, Karthik},
  booktitle={Advances in Neural Information Processing Systems}, volume={36}, year={2023},
  url={https://proceedings.neurips.cc/paper/2023/hash/271db9922b8d1f4dd7aaef84ed5ac703-Abstract.html}
}

@inproceedings{reflexion2023,
  title={Reflexion: Language Agents with Verbal Reinforcement Learning},
  author={Shinn, Noah and Cassano, Federico and Gopinath, Ashwin and Narasimhan, Karthik and Yao, Shunyu},
  booktitle={Advances in Neural Information Processing Systems}, volume={36}, year={2023},
  url={https://proceedings.neurips.cc/paper_files/paper/2023/hash/1b44b878bb782e6954cd888628510e90-Abstract-Conference.html}
}

@inproceedings{debate2024,
  title={Improving Factuality and Reasoning in Language Models through Multiagent Debate},
  author={Du, Yilun and Li, Shuang and Torralba, Antonio and Tenenbaum, Joshua B. and Mordatch, Igor},
  booktitle={Proceedings of the 41st International Conference on Machine Learning},
  series={Proceedings of Machine Learning Research}, volume={235}, pages={11733--11763}, year={2024},
  url={https://proceedings.mlr.press/v235/du24e.html}
}

@inproceedings{leasttomost2023,
  title={Least-to-Most Prompting Enables Complex Reasoning in Large Language Models},
  author={Zhou, Denny and Sch{\"a}rli, Nathanael and Hou, Le and Wei, Jason and Scales, Nathan and Wang, Xuezhi and Schuurmans, Dale and Cui, Claire and Bousquet, Olivier and Le, Quoc V. and Chi, Ed H.},
  booktitle={International Conference on Learning Representations}, year={2023},
  url={https://openreview.net/forum?id=WZH7099tgfM}
}

@inproceedings{bge2024,
  title={{C-Pack}: Packed Resources For General Chinese Embeddings},
  author={Xiao, Shitao and Liu, Zheng and Zhang, Peitian and Muennighoff, Niklas and Lian, Defu and Nie, Jian-Yun},
  booktitle={Proceedings of the 47th International ACM SIGIR Conference on Research and Development in Information Retrieval},
  pages={641--649}, year={2024}, doi={10.1145/3626772.3657878},
  url={https://sigir-2024.github.io/proceedings.html}
}

@misc{bge_modelcard,
  title={{BAAI/bge-large-en-v1.5}: Model Card},
  author={{Beijing Academy of Artificial Intelligence}}, year={2023},
  url={https://huggingface.co/BAAI/bge-large-en-v1.5}, note={Accessed September 16, 2026}
}

@inproceedings{mteb2023,
  title={{MTEB}: Massive Text Embedding Benchmark},
  author={Muennighoff, Niklas and Tazi, Nouamane and Magne, Lo{\"i}c and Reimers, Nils},
  booktitle={Proceedings of the 17th Conference of the European Chapter of the Association for Computational Linguistics},
  pages={2014--2037}, year={2023}, doi={10.18653/v1/2023.eacl-main.148},
  url={https://aclanthology.org/2023.eacl-main.148/}
}

@inproceedings{gpqa2024,
  title={{GPQA}: A Graduate-Level Google-Proof {Q\&A} Benchmark},
  author={Rein, David and Hou, Betty Li and Stickland, Asa Cooper and Petty, Jackson and Pang, Richard Yuanzhe and Dirani, Julien and Michael, Julian and Bowman, Samuel R.},
  booktitle={Conference on Language Modeling}, year={2024},
  url={https://openreview.net/forum?id=Ti67584b98}
}

@article{codecontests2022,
  title={Competition-Level Code Generation with {AlphaCode}},
  author={Li, Yujia and Choi, David and Chung, Junyoung and Kushman, Nate and Schrittwieser, Julian and Leblond, R{\'e}mi and Eccles, Tom and Keeling, James and Gimeno, Felix and Dal Lago, Agustin and Hubert, Thomas and Choy, Peter and de Masson d'Autume, Cyprien and Babuschkin, Igor and Chen, Xinyun and Huang, Po-Sen and Welbl, Johannes and Gowal, Sven and Cherepanov, Alexey and Molloy, James and Mankowitz, Daniel J. and Robson, Esme Sutherland and Kohli, Pushmeet and de Freitas, Nando and Kavukcuoglu, Koray and Vinyals, Oriol},
  journal={Science}, volume={378}, number={6624}, pages={1092--1097}, year={2022},
  doi={10.1126/science.abq1158}, url={https://arxiv.org/abs/2203.07814}
}

@inproceedings{livecodebench2024,
  title={{LiveCodeBench}: Holistic and Contamination Free Evaluation of Large Language Models for Code},
  author={Jain, Naman and Han, King and Gu, Alex and Li, Wen-Ding and Yan, Fanjia and Zhang, Tianjun and Wang, Sida and Solar-Lezama, Armando and Sen, Koushik and Stoica, Ion},
  booktitle={International Conference on Learning Representations}, year={2025},
  url={https://proceedings.iclr.cc/paper_files/paper/2025/hash/94074dd5a072d28ff75a76dabed43767-Abstract-Conference.html}
}

@misc{aime_source,
  title={American Invitational Mathematics Examination},
  author={{Mathematical Association of America}}, year={2026},
  url={https://maa.org/maa-invitational-competitions/}, note={Official competition resource; accessed September 16, 2026}
}

@inproceedings{swebench2024,
  title={{SWE-bench}: Can Language Models Resolve Real-World {GitHub} Issues?},
  author={Jimenez, Carlos E. and Yang, John and Wettig, Alexander and Yao, Shunyu and Pei, Kexin and Press, Ofir and Narasimhan, Karthik},
  booktitle={International Conference on Learning Representations}, year={2024},
  url={https://proceedings.iclr.cc/paper_files/paper/2024/hash/edac78c3e300629acfe6cbe9ca88fb84-Abstract-Conference.html}
}

@article{scoreflow2025,
  title={{ScoreFlow}: Mastering {LLM} Agent Workflows via Score-based Preference Optimization},
  author={Wang, Yinjie and Yang, Ling and Li, Guohao and Wang, Mengdi and Aragam, Bryon},
  journal={arXiv preprint arXiv:2502.04306}, year={2025},
  url={https://arxiv.org/abs/2502.04306}
}

@article{flowreasoner2025,
  title={{FlowReasoner}: Reinforcing Query-Level Meta-Agents},
  author={Gao, Hongcheng and Liu, Yue and He, Yufei and Dou, Longxu and Du, Chao and Deng, Zhijie and Hooi, Bryan and Lin, Min and Pang, Tianyu},
  journal={arXiv preprint arXiv:2504.15257}, year={2025},
  url={https://arxiv.org/abs/2504.15257}
}

@inproceedings{masrouter2025,
  title={{MasRouter}: Learning to Route {LLM}s for Multi-Agent Systems},
  author={Yue, Yanwei and Zhang, Guibin and Liu, Boyang and Wan, Guancheng and Wang, Kun and Cheng, Dawei and Qi, Yiyan},
  booktitle={Proceedings of the 63rd Annual Meeting of the Association for Computational Linguistics (Volume 1: Long Papers)},
  pages={15549--15572}, publisher={Association for Computational Linguistics}, year={2025},
  doi={10.18653/v1/2025.acl-long.757}, url={https://aclanthology.org/2025.acl-long.757/}
}

@article{frugalgpt2024,
  title={{FrugalGPT}: How to Use Large Language Models While Reducing Cost and Improving Performance},
  author={Chen, Lingjiao and Zaharia, Matei and Zou, James},
  journal={Transactions on Machine Learning Research}, year={2024},
  url={https://openreview.net/forum?id=cSimKw5p6R}
}

@article{agentless2025,
  title={Demystifying {LLM}-Based Software Engineering Agents},
  author={Xia, Chunqiu Steven and Deng, Yinlin and Dunn, Soren and Zhang, Lingming},
  journal={Proceedings of the ACM on Software Engineering}, volume={2}, number={FSE},
  articleno={FSE037}, numpages={24}, year={2025},
  doi={10.1145/3715754}, url={https://doi.org/10.1145/3715754}
}

@article{qwen3embedding2025,
  title={{Qwen3} Embedding: Advancing Text Embedding and Reranking Through Foundation Models},
  author={Zhang, Yanzhao and Li, Mingxin and Long, Dingkun and Zhang, Xin and Lin, Huan and Yang, Baosong and Xie, Pengjun and Yang, An and Liu, Dayiheng and Lin, Junyang and Huang, Fei and Zhou, Jingren},
  journal={arXiv preprint arXiv:2506.05176}, year={2025},
  url={https://arxiv.org/abs/2506.05176}
}
\bibliographystyle{iclr2027_conference}
\clearpage
\appendix
\section{Proofs and effective boundaries}
\label{app:proofs}

The analysis separates three questions: what a finite probe panel identifies,
when an estimated representation preserves a decision, and when a probe's
measured advantage survives conversion to an executable expert. The arguments
adapt standard concentration and decision-theoretic tools; they are not claims
of new concentration inequalities or a universal sufficiency theorem. Their
purpose is to identify assumptions that an implementation and an evaluation
must satisfy. This perspective follows the distinction between an instance
representation and the algorithm selected from it \citep{rice1976}. The analogy
to task-relevant representation learning \citep{zhang2021bisimulation} concerns
preserving decisions; our finite vector is not a bisimulation representation.

\subsection{Probability space and fixed-protocol convention}
\label{app:probability-space}

Let $Q\sim\mathcal D$ and fix an executable panel
$\mathcal P=\{p_1,\ldots,p_K\}$. A protocol specifies the model and endpoint,
prompts, available tools, answer extraction, score, execution limits, and failure
handling. Under this protocol, $Y_{k,s}(q)\in[0,1]$ is the score of execution $s$
of probe $k$, $f_k(q)=\mathbb E[Y_{k,s}(q)\mid q]$, and
$\bar f_k(q)=S_k^{-1}\sum_{s=1}^{S_k}Y_{k,s}(q)$.
The conditional expectation averages execution randomness, not changes to the
protocol. Conditioning on successful retries defines a different estimand from
assigning failed executions score zero.

Throughout the routing analysis, normalization statistics, centroids, the
encoder, and the deployed expert library are fixed, or we condition on the
training data and randomness that produced them. Evaluation queries are fresh.
Write $\Lambda=\operatorname{diag}(\sigma_1,\ldots,\sigma_K)$, where every
$\sigma_k>0$, and let
$z(q)=\Lambda^{-1}(f(q)-\mu)$ and
$\bar z(q)=\Lambda^{-1}(\bar f(q)-\mu)$.
Thus $z$ is the population embedding and $\bar z$ its measured estimate.
The encoder predicts $\hat z(q)=g_\theta(\psi(q))$; write
$x_i=\psi(q_i)$ for an individual example and $X=\psi(Q)$ for a fresh
random query. These objects remain distinct even when measured and predicted
embeddings lead to the same expert.
The $J$ centroids are pairwise distinct and use a fixed tie rule. If a query
lies on a nearest-centroid tie, its margin is zero. We write
$R_w(q)=\mathbb E[\operatorname{score}(w,q)\mid q]\in[0,1]$ for the
expected reward of an executable expert. Equality of selected experts implies
equality of these expected rewards, not of independent realized scores.

\subsection{Measurement accuracy and the information in a finite panel}
\label{app:measurement-proof}

\begin{proposition}[Affinity estimation preserves probe decisions]
\label{prop:measurement}
Fix a query and protocol. If each probe's $S$ repetitions are independent, identically distributed, and bounded in $[0,1]$, then
\begin{equation}
\Pr\!\left(\|\bar f(q)-f(q)\|_\infty>\epsilon\right)
\le 2K e^{-2S\epsilon^2}.
\label{eq:measurement}
\end{equation}
On the event $\|\bar f-f\|_\infty\le\epsilon$, selecting $\hat k\in\argmax_k\bar f_k(q)$ incurs expected probe regret at most $2\epsilon$.
\end{proposition}

\begin{proof}[Proof of Proposition~\ref{prop:measurement}]
For each fixed $q$ and $k$, independent bounded executions give
\begin{equation}
 \Pr\bigl(|\bar f_k(q)-f_k(q)|>\epsilon\mid q\bigr)
 \le 2\exp(-2S\epsilon^2)
\end{equation}
by Hoeffding's inequality \citep{hoeffding1963}. A union bound over $K$
coordinates proves the stated $2K\exp(-2S\epsilon^2)$ bound. Independence
between different probes is not required. Let $k^*\in\arg\max_k f_k(q)$ and
$\hat k\in\arg\max_k\bar f_k(q)$. On the uniform accuracy event,
\begin{align}
 f_{k^*}(q)-f_{\hat k}(q)
 &= f_{k^*}(q)-\bar f_{k^*}(q)
   +\bar f_{k^*}(q)-\bar f_{\hat k}(q)
   +\bar f_{\hat k}(q)-f_{\hat k}(q)\\
 &\le \epsilon+0+\epsilon.
\end{align}
Thus a unique best probe with population gap exceeding $2\epsilon$ is recovered.
For transfer between two queries, replace $\bar f(q)$ in this argument by
$f(q')$. If $\|f(q)-f(q')\|_\infty\le\epsilon$, using a best probe for
$q'$ on $q$ loses at most $2\epsilon$ relative to the best probe for $q$.
\end{proof}

The same argument connects standardized distance to probe transfer. Let
$\sigma_{\max}=\max_k\sigma_k$. Since
$f(q)-f(q')=\Lambda(z(q)-z(q'))$, the condition
$\|z(q)-z(q')\|_2\le\delta$ implies
$\|f(q)-f(q')\|_\infty\le\sigma_{\max}\delta$.
Thus the expected score lost by transferring the best probe is at most
$2\sigma_{\max}\delta$. The scale factor is necessary because distances in
standardized coordinates are not raw score differences.

For unequal repeat counts the failure bound is
$2\sum_k\exp(-2S_k\epsilon^2)$; for $N$ prespecified queries with equal
repeat counts, a further union bound replaces $K$ by $NK$. In particular,
$\epsilon_S(\delta)=\sqrt{\log(2K/\delta)/(2S)}$ is a sufficient
single-query radius at confidence $1-\delta$. For $K=12$, $S=3$, and
$\delta=0.05$, this radius is approximately $1.014$, hence it does not certify
accurate rankings for scores in $[0,1]$. The result explains the role of
repetitions but is not a nonvacuous numerical certificate for our three-repeat
fingerprints.

\paragraph{Standardization changes both noise and the decision geometry.}
For fixed train-fitted $\Lambda$,
\begin{equation}
 \|\bar z-z\|_2
 \le \|\bar f-f\|_\infty
       \left(\sum_{k=1}^K\sigma_k^{-2}\right)^{1/2}.
 \label{eq:normalization-noise}
\end{equation}
Small coordinate scales can therefore amplify execution noise. Positive
diagonal scaling retains the information in the raw vector, but changes
Euclidean neighborhoods and cross-coordinate rankings. Our code substitutes
one for nearly zero fitted standard deviations; it does not otherwise cap
this amplification. Equation~\eqref{eq:normalization-noise} conditions on
normalization fitted without the evaluated query. It is not a joint guarantee
for a training observation reused to fit its own scale, and it does not cover
nonlinear per-query unit normalization.

The finite-panel guarantee has a precise information boundary. The vector
$f(q)$ determines the expected reward of each existing probe and of any
randomized choice among them: conditional on fixed mixture weights
$a\in\Delta_K$, that reward is $a^\top f(q)$. Consequently,
$\|f(q)-f(q')\|_\infty$ is a task-dependent pseudometric for these decisions.
It need not distinguish queries requiring different unseen compositions.

\begin{proposition}[Finite-panel indistinguishability]
\label{prop:finite-panel-limit}
There exist a probe panel, a query distribution, and a two-expert bank such
that the queries have identical population fingerprints, but every selector
from that bank using only the fingerprint has positive regret, even with
infinitely many probe executions.
\end{proposition}
\begin{proof}
Let $Q$ be uniform on $\{q_0,q_1\}$ and set
$f_k(q_0)=f_k(q_1)=1/2$ for every $k$. Let the selectable expert bank consist
of $w_0$ and $w_1$, with rewards
$R_{w_0}(q_0)=R_{w_1}(q_1)=1$ and
$R_{w_0}(q_1)=R_{w_1}(q_0)=0$. A selector observing only $f(Q)$ sees the
same input for both queries. If it chooses $w_0$ with probability $a$, its
mean reward is $(a+(1-a))/2=1/2$, whereas a query-aware oracle achieves one.
The regret is $1/2$ for every $a$. Standardization and perfect estimation of
$f$ do not alter this example.
\end{proof}

An extension to a larger workflow class therefore needs an extra assumption.
For example, suppose $|R_w(q)-h_w(z(q))|\le\eta$ for all $q,w$, every
$h_w$ is $L$-Lipschitz, and a selector is within $\kappa$ of maximizing
$h_w(\tilde z)$. Adding and subtracting the two surrogate rewards gives
\begin{equation}
 \sup_w R_w(q)-R_{\tilde w}(q)
 \le 2\eta+2L\|\tilde z-z(q)\|_2+\kappa.
 \label{eq:conditional-sufficiency}
\end{equation}
The substantive assumption is the factorization of rewards through $z$;
neither finite-panel concentration nor good probe prediction establishes it.
We do not assume this factorization in our main results.

\subsection{Route stability and the information available to an encoder}
\label{app:routing-proof}

\begin{proposition}[A margin condition for faithful assignment]
\label{prop:routing}
Condition on fixed distinct centroids and a fixed encoder. For a uniquely assigned point $z$, let $a=r_C(z)$ and
\begin{equation}
m_C(z)=\min_{b\ne a}
\frac{\|z-c_b\|_2^2-\|z-c_a\|_2^2}{2\|c_b-c_a\|_2}.
\label{eq:margin}
\end{equation}
Set $m_C(z)=0$ on ties. If $\|\hat z-z\|_2<m_C(z)$, the assignments agree. With $\mathcal E=\E\|g_\theta(X)-z(Q)\|_2^2$, for every $t>0$,
\begin{equation}
\Pr\{r_C(g_\theta(X))\ne r_C(z(Q))\}
\le\Pr\{m_C(z(Q))\le t\}+\mathcal E/t^2.
\label{eq:routing}
\end{equation}
\end{proposition}

\begin{proof}[Proof of Proposition~\ref{prop:routing}]
Set $a=r_C(z)$ and write $\tilde z=z+e$. For every competing centroid $b$,
\begin{align}
 \|\tilde z-c_b\|_2^2-\|\tilde z-c_a\|_2^2
 &=\|z-c_b\|_2^2-\|z-c_a\|_2^2+2e^\top(c_a-c_b)\\
 &\ge 2\|c_b-c_a\|_2\bigl(m_C(z)-\|e\|_2\bigr).
\end{align}
If $\|e\|_2<m_C(z)$, all comparisons are strictly positive, so the
nearest centroid is unchanged. With $Z=z(Q)$, $X=\psi(Q)$, and
$\mathcal E=\mathbb E\|g_\theta(X)-Z\|_2^2$, a disagreement can occur
only if $m_C(Z)\le t$ or $\|g_\theta(X)-Z\|_2\ge t$. The union bound
and Markov's inequality give
\begin{equation}
 \Pr\{r_C(g_\theta(X))\ne r_C(Z)\}
 \le \Pr\{m_C(Z)\le t\}+\mathcal E/t^2.
 \label{eq:routing-probability}
\end{equation}
Finally, two selected experts have the same expected reward when their
indices agree and a reward difference of magnitude at most one otherwise.
Taking expectations proves the reward-difference bound.
\end{proof}

The simpler nearest-versus-second-nearest distance gap
$\Delta_C(z)$ gives the sufficient condition
$2\|\tilde z-z\|_2<\Delta_C(z)$ by the triangle inequality. The exact
Voronoi margin can be less conservative. Neither condition states that the
centroids recover optimal workflow classes or that a preserved route selects
a high-quality expert. In particular, the measured-fingerprint route is not
an upper bound on the predicted route's reward. When multiple indices execute
the same specification, index disagreement also overstates disagreement
between executable behaviors.

\paragraph{Measured and population routes differ.}
Writing $\bar Z=\bar z(Q)$, both sources of error matter when comparing
the measured route with the encoder:
\begin{align}
 \Pr\{r_C(\bar Z)\ne r_C(g_\theta(X))\}
 &\le\Pr\{m_C(Z)\le t\}
   +\Pr\{\|\bar Z-Z\|_2\ge t\}\\
 &\quad+\Pr\{\|g_\theta(X)-Z\|_2\ge t\}.
 \label{eq:measured-route}
\end{align}
Alternatively, $\|g_\theta(X)-\bar Z\|_2<m_C(\bar Z)$ certifies
agreement with the measured route for that observation. This is a
deterministic diagnostic, not a certificate of population routing accuracy.

\begin{proposition}[Prediction noise and conditional information]
\label{prop:prediction-floor}
Assume finite second moments, fixed affine normalization, and
$\bar Z=Z+\xi$ with $\mathbb E[\xi\mid Q]=0$. For any fixed predictor
$g$ and $g^*(X)=\mathbb E[Z\mid X]$,
\begin{align}
 \mathbb E\|g(X)-\bar Z\|_2^2
 &=\mathbb E\|g(X)-g^*(X)\|_2^2
   +\mathbb E\operatorname{tr}\operatorname{Cov}(Z\mid X)
   +\mathbb E\|\xi\|_2^2. \label{eq:prediction-decomposition}
\end{align}
For the population route $L=r_C(Z)$, any selector $\hat L$ using only $X$
satisfies
\begin{equation}
 \Pr(\hat L\ne L)
 \ge \mathbb E\left[1-\max_j\Pr(L=j\mid X)\right].
 \label{eq:route-information-floor}
\end{equation}
\end{proposition}
\begin{proof}
Expand $g(X)-\bar Z=(g(X)-g^*(X))+(g^*(X)-Z)-\xi$.
The first two summands are orthogonal in expectation by conditional-mean
orthogonality. The noise is orthogonal to both because $X$ and $Z$ are
functions of $Q$ and $\mathbb E[\xi\mid Q]=0$. The middle squared norm
equals the expected trace of the conditional covariance, proving
\eqref{eq:prediction-decomposition}. Conditional on $X=x$, any deterministic
prediction succeeds with probability at most $\max_j\Pr(L=j\mid X=x)$.
A randomized prediction is a convex combination of these success
probabilities and obeys the same bound. Average over $X$.
\end{proof}

For independent repeats within each coordinate,
\begin{equation}
 \mathbb E[\|\xi\|_2^2\mid Q=q]
 =\sum_{k=1}^K\frac{\operatorname{Var}(Y_k(q)\mid q)}{S_k\sigma_k^2}
 \le\sum_{k=1}^K\frac{1}{4S_k\sigma_k^2}.
 \label{eq:noise-variance}
\end{equation}
Cross-coordinate dependence does not affect this sum of squared coordinate
errors. More executions can reduce measurement noise, but cannot remove the
conditional-information term for a fixed $\psi$. That term can nevertheless
be zero; a frozen representation does not by itself establish a positive
information floor. Our mixed cosine/MSE training objective is not pure
squared-loss minimization, so conditional-mean optimality does not claim that
the trained network attains $g^*$. Reported coordinate-averaged MSE must be
multiplied by $K$ to match the vector-error convention above.

\subsection{From measured probe preference to an executable seed}
\label{app:seed-proof}

\begin{proposition}[Utility of an affinity-derived initialization]
\label{prop:seed}
Fix a cluster. Let $r_k$ be probe $k$'s expected raw score and $s_k$ its estimate, with $\max_k|s_k-r_k|\le\epsilon$. Let $\mathcal R\subseteq\{1,\ldots,K\}$ be the nonempty set of runnable probe indices and $\hat k\in\mathcal R$ the chosen probe. Define
$\alpha=\max_k r_k-\max_{k\in\mathcal R}r_k$ and
$\beta=\max_{k\in\mathcal R}s_k-s_{\hat k}$.
If its mapped seed $v_{\hat k}$ has expected cluster utility $U(v_{\hat k})\ge r_{\hat k}-\tau$, then
\begin{equation}
U(v_{\hat k})\ge\max_k r_k-\alpha-\beta-2\epsilon-\tau.
\label{eq:seed}
\end{equation}
\end{proposition}

\begin{proof}[Proof of Proposition~\ref{prop:seed}]
Suppress the cluster index. Let $\hat k\in\mathcal R$ be the selected
runnable probe and $v_{\hat k}$ its executable seed. On the event
$\max_k|s_k-r_k|\le\epsilon$, use the definitions
$\alpha=\max_k r_k-\max_{k\in\mathcal R}r_k$ and
$\beta=\max_{k\in\mathcal R}s_k-s_{\hat k}$ to obtain
\begin{align}
 \max_k r_k
 &=\alpha+\max_{k\in\mathcal R}r_k\\
 &\le\alpha+\max_{k\in\mathcal R}s_k+\epsilon
 =\alpha+s_{\hat k}+\beta+\epsilon\\
 &\le\alpha+r_{\hat k}+\beta+2\epsilon
 \le\alpha+U(v_{\hat k})+\beta+2\epsilon+\tau.
\end{align}
Rearranging proves the claim. Only the one-sided transfer condition
$U(v_{\hat k})\ge r_{\hat k}-\tau$ for the chosen seed is needed.
\end{proof}

This form allows the implemented seed to maximize standardized scores rather
than raw utility: $\hat k\in\arg\max_{k\in\mathcal R}
(s_k-\mu_k)/\sigma_k$. The raw-score deficit $\beta$ is therefore
essential. Even perfect measurement need not make it zero. For example,
raw cluster means $(0.9,0.6)$, fitted global means $(0.89,0.3)$, and scales
$(0.1,0.1)$ yield standardized means $(0.1,3)$, selecting the second probe
and incurring a raw deficit $0.3$. This is a deliberate distinction between
relative cluster preference and absolute reward; a standardized seed rule
needs empirical justification at this interface.

Likewise, $\tau$ cannot be omitted because probe names do not specify
executable equivalence. The current mapping merges several probe families,
omits tool-use and ReAct executors, and adds a chain-of-thought instruction to
every guided seed. The proposition does not assume these transformations
preserve reward. The coverage term $\alpha$ measures unavailable panel
capabilities; it does not measure the gap to all possible workflows.

\paragraph{Sampling a cluster profile.}
Suppose a fixed cluster has $n$ fresh queries sampled independently from its
conditional query distribution, each with $S$
independent executions per probe, and let
$s_k=n^{-1}\sum_i\bar f_k(Q_i)$ and
$r_k=\mathbb E[f_k(Q)\mid Q\text{ belongs to the cluster}]$.
Separating query sampling from execution noise gives
\begin{equation}
 \Pr\{\max_k|s_k-r_k|>a+b\}
 \le 2K e^{-2na^2}+2K e^{-2nSb^2}.
 \label{eq:profile-concentration}
\end{equation}
The first term applies Hoeffding to $f_k(Q_i)$; the second applies it
conditionally on the queries to the $nS$ centered execution scores, assuming
conditional independence across these executions. More repeated executions
do not eliminate finite-query uncertainty. This argument does not directly
apply when cluster membership is selected using those same noisy executions,
as in clustering measured training fingerprints.

\subsection{What validation and operator priors can guarantee}
\label{app:selection}

\begin{lemma}[Seed retention under accurate selection]
\label{lem:seed-retention}
Let a candidate set $\mathcal V$ contain seed $v$, with utilities $U(w)$
and recorded estimates $\widehat U(w)$. Suppose
$\sup_{w\in\mathcal V}|\widehat U(w)-U(w)|\le\eta$ and
$\widehat U(w_{\mathrm{out}})\ge
\max_{w\in\mathcal V}\widehat U(w)-h$. Then
\begin{equation}
 U(w_{\mathrm{out}})\ge\max_{w\in\mathcal V}U(w)-2\eta-h
 \ge U(v)-2\eta-h.
 \label{eq:seed-retention}
\end{equation}
\end{lemma}
\begin{proof}
For every $w\in\mathcal V$,
$U(w_{\mathrm{out}})\ge\widehat U(w_{\mathrm{out}})-\eta
\ge\widehat U(w)-h-\eta\ge U(w)-h-2\eta$.
Maximize over $w$.
\end{proof}

For $M$ candidates frozen independently of $n$ independently sampled held-out
queries, with independent execution randomness across queries, bounded query
scores and a union bound yield the sufficient radius
$\eta=\sqrt{\log(2M/\delta)/(2n)}$. Candidate scores may be dependent
across workflows. Repeating executions on the same $n$ queries does not turn
the population sample size into $nS$. Importantly, substituting the number
of \emph{observed} candidates after adapting them to validation failures does
not justify this finite-class guarantee. It requires a prespecified class,
fresh validation with appropriate error allocation, or a separate final
holdout. An intact external test evaluates the selected procedure but does
not retroactively certify every accepted edit.

The implemented acceptance rule retains the incumbent unless a proposal's
recorded score exceeds it by a margin $h$. With fixed recorded scores, the
final incumbent is within $h$ of the best evaluated score: a rejected
candidate was at most its then-incumbent plus $h$, and incumbent scores only
increase. This is an empirical property. The margin $2/n$ used in our search
is not a calibrated confidence radius. An accepted edit is guaranteed to
improve population utility under uniform error $\eta$ only when its measured
increment exceeds $2\eta$. Our adaptive validation protocol has no such
established uniform certificate.

\paragraph{A hard prior introduces an exclusion cost.}
Let $U_j(w)$ be expected utility on a fixed cluster $j$, and let
$\mathcal H$ denote a search class, distinct from the retained expert bank
$\mathcal W$. For a permitted class
$\mathcal H_j^{\mathrm{prior}}\subseteq\mathcal H$, define
$L_j=\sup_{w\in\mathcal H}U_j(w)
-\sup_{w\in\mathcal H_j^{\mathrm{prior}}}U_j(w)$.
If every excluded workflow can be transformed into a permitted one while
losing at most $\zeta_j$, then $L_j\le\zeta_j$, by applying the
transformation to an arbitrarily near-optimal workflow. Low standalone probe
reward is insufficient for this condition: two operators can each have
reward zero alone and reward one together. The seed exemption ensures an
available starting workflow; it does not establish that an operator ban is
harmless. Similarly, if each guided proposal hits a target set with
conditional probability at least $p$ given any no-hit history, $T$ proposals
hit it with probability at least $1-(1-p)^T$, by multiplying conditional
no-hit probabilities. A speedup requires evidence that guidance improves
this probability under a matched budget; using fingerprints in a prompt
alone does not supply that evidence.

\subsection{A conditional bound for the complete pipeline}
\label{app:composed-bound}

The preceding statements can be combined without assuming that the panel
is sufficient for all generated workflows. The comparator is explicitly the
population best \emph{probe}, not an unrestricted workflow oracle.
Let $L=r_C(z(Q))$, $\pi_j=\Pr(L=j)$, and for nonempty clusters define
$r_{jk}=\mathbb E[f_k(Q)\mid L=j]$ and
$U_j(w)=\mathbb E[R_w(Q)\mid L=j]$. The population assignment $L$ uses
expected affinities; it is distinct from experimental assignment A, which
uses measured test-query affinities. Define the cluster heterogeneity cost
\begin{equation}
 H_C=\mathbb E\max_k f_k(Q)-\sum_{j=1}^J\pi_j\max_k r_{jk}\ge0.
 \label{eq:niche-heterogeneity}
\end{equation}
Nonnegativity follows by interchanging a maximum and a conditional
expectation. This is the price, relative to per-query probe selection, of
requiring one probe decision per population cluster; generated experts may
recover or exceed this lost value.

\begin{theorem}[Conditional composition relative to the probe oracle]
\label{thm:composed-bound}
Fix the protocol, normalization, centroids, encoder, and expert library.
For each cluster with $\pi_j>0$, assume the profile-accuracy and seed-transfer
conditions of Proposition~\ref{prop:seed} hold with parameters
$(\alpha_j,\beta_j,\epsilon_j,\tau_j)$ for the population cluster $L=j$.
Assume the final expert $w_j$ satisfies Lemma~\ref{lem:seed-retention}
with error $\eta_j$ and selection deficit $h_j$. Set
\begin{equation}
 B_j=\alpha_j+\beta_j+2\epsilon_j+\tau_j+2\eta_j+h_j.
\end{equation}
On these simultaneous accuracy events, for every $t>0$,
\begin{align}
 &\mathbb E\left[\max_k f_k(Q)
       -R_{w_{r_C(g_\theta(X))}}(Q)\right]\\
 &\qquad\le H_C+\sum_{j=1}^J\pi_j B_j
       +\Pr\{m_C(z(Q))\le t\}
       +\frac{\mathbb E\|g_\theta(X)-z(Q)\|_2^2}{t^2}.
 \label{eq:composed-bound}
\end{align}
All expectations and probabilities on this event are over a fresh query
under the fixed fitted objects. Empty clusters contribute zero.
\end{theorem}
\begin{proof}
Proposition~\ref{prop:seed} and Lemma~\ref{lem:seed-retention} imply
$U_j(w_j)\ge\max_k r_{jk}-B_j$. Hence
\begin{align}
 \mathbb E R_{w_L}(Q)
 &=\sum_j\pi_jU_j(w_j)
 \ge\sum_j\pi_j\max_k r_{jk}-\sum_j\pi_jB_j\\
 &=\mathbb E\max_k f_k(Q)-H_C-\sum_j\pi_jB_j.
\end{align}
Replacing $L$ by the predicted route reduces expected reward by at most
their disagreement probability. Apply
Equation~\eqref{eq:routing-probability} and rearrange.
\end{proof}

This theorem distinguishes within-cluster heterogeneity, missing executable
capabilities, standardized-ranking distortion, probe-to-seed transfer,
selection uncertainty, and encoder decision error. It does not establish
that any term is small. In particular, the observed clusters were fitted to
noisy fingerprints, the measured profiles do not directly sample the latent
clusters in the theorem, and adaptive search does not provide the stated
selection event. Endpoint or scoring changes also change $f$ and $R_w$.
The result therefore specifies effective boundaries and targets for
validation, rather than certifying the existing pipeline or comparing it to
every possible generated workflow.

\section{Implementation and design rationale}
\label{app:method}
\subsection{Probe panel and its interpretation}
The panel is a set of concrete interventions on one executor, not a claim that the names of reasoning methods form an orthogonal or complete basis. Table~\ref{tab:probe-map} distinguishes the measured implementations from their generator-side approximations. Related publications motivate the computational families; the local prompts, budgets, and feedback contracts define the actual experiments. In particular, the measured tool-use and ReAct probes have execution access that the guided interpreter lacks, and the generator's persona branch does not reproduce the measured three-agent debate.

\begin{table}[t]
\caption{Measured probe families and their mapping into the guided interpreter. A mapping is an implementation approximation, not an equivalence assertion. CoT is always used for the interpreter's base candidates.}
\label{tab:probe-map}
\centering\small
\begin{tabularx}{\linewidth}{@{}p{.13\linewidth}>{\raggedright\arraybackslash}X>{\raggedright\arraybackslash}X@{}}
\toprule
Probe & Measured computational intervention & Guided interpreter mapping \\
\midrule
Direct & One formatted response & Direct flag is inert; base candidate still uses CoT. \\
CoT & Step-by-step reasoning prompt \citep{cot2022} & CoT base candidate. \\
CoT-SC & Five inner samples plus answer aggregation \citep{selfconsistency2023} & Sampling and selection; seed uses three samples. \\
ToT & Bounded branching, evaluation and search \citep{tot2023} & Sampling and selection; no tree search. \\
Self-refine & Critique and revision \citep{selfrefine2023} & One checking/rewrite call. \\
Reflexion & Outcome feedback, reflection and retry \citep{reflexion2023} & Same revision flag; no separate feedback-memory loop. \\
RAG & Same-benchmark TF-IDF exemplar retrieval; query excluded & CoT flag; no retrieval in generated expert. \\
ReAct & Bounded action--observation interaction \citep{react2023} & Unavailable; cannot seed its behavior. \\
Tool use & Python execution or public-test feedback, domain dependent & Unavailable; cannot seed its behavior. \\
Debate & Independent debaters and a judge \citep{debate2024} & Two fixed persona candidates plus selection. \\
Role team & Three complementary roles and a coordinator & Same two-persona branch as debate. \\
Decompose & Subproblem decomposition and synthesis \citep{leasttomost2023} & One plan prepended to subsequent candidate calls. \\
\bottomrule
\end{tabularx}
\end{table}

The canonical mixed-task panel contains $K=12$ probes, and each fingerprint coordinate averages $S=3$ \emph{outer} repetitions. For the CoT-SC probe, the five inner samples occur within each repetition and must not be counted as five independent fingerprint observations. The score retains its task-specific meaning; fractional coding-test utility is not a calibrated probability of solving an unseen hidden test set. All normalizations preserve the raw coordinates in the stored artifact.

The relative-normalization variant subtracts a query's mean outcome and divides by the resulting vector norm; this changes the estimand and removes absolute difficulty information. It does not inherit the affine-noise decomposition in Equation~\eqref{eq:prediction-decomposition}. Comparing raw, standardized, and relative coordinates is thus a test of a design tradeoff, not just numerical preprocessing. The default affine standardization is specified below.

\subsection{Encoder and cluster construction}
Let $\mathcal T=\{q_i\}_{i=1}^{N}$ denote the construction pool and
$\mathcal T_{\mathrm{enc}}\subseteq\mathcal T$ the encoder's outer-training
subset, with $N_{\mathrm{enc}}=|\mathcal T_{\mathrm{enc}}|$.
Each query has cached probe outcomes. Write
$\bar f(q)=(\bar f_1(q),\ldots,\bar f_K(q))^\top$ for its measured mean
scores from Equation~\eqref{eq:affinity}. We fit the coordinate means
$\mu_k$ and standard deviations $\sigma_k$ on $\mathcal T_{\mathrm{enc}}$
and define the measured embedding by
\begin{equation}
\bar z_k(q)=\frac{\bar f_k(q)-\mu_k}{\sigma_k},\qquad
\bar z(q)=\Lambda^{-1}\bigl(\bar f(q)-\mu\bigr),
\quad \Lambda=\operatorname{diag}(\sigma_1,\ldots,\sigma_K).
\label{eq:standardization}
\end{equation}
Here $\mu=(\mu_1,\ldots,\mu_K)^\top$, and $\Lambda$ has positive diagonal
scales. A fitted standard deviation below $10^{-8}$ is replaced by one,
avoiding division by zero; a nonconstant but low-variance coordinate can
still receive a large weight. Positive standardized values indicate scores
above that probe's training mean, not probabilities. The fitted normalization
is used unchanged for all remaining construction fingerprints and any later
measured test fingerprints.

The encoder predicts $\hat z(q)=g_\theta(\psi(q))$, where $\psi$ is the
frozen text feature extractor and $g_\theta$ is the trainable head with
parameters $\theta$. Its target is $\bar z(q)$. The training objective is
\begin{equation}
\mathcal L(\theta)=\frac{1}{N_{\mathrm{enc}}}\sum_{q\in\mathcal T_{\mathrm{enc}}}
\left[\lambda\bigl(1-\cos(\hat z(q),\bar z(q))\bigr)
+\frac{1-\lambda}{K}\|\hat z(q)-\bar z(q)\|_2^2\right],
\quad \lambda=0.7.
\label{eq:loss}
\end{equation}
Here $\cos(u,v)$ denotes cosine similarity. Its loss term encourages
agreement in the relative pattern of probe responses; the second term
penalizes coordinate error in the space used for cluster assignment.
Dividing by $K$ makes the squared error a per-coordinate mean. The objective
depends on $\theta$ through $\hat z(q)$, and $\lambda$ is a fixed
implementation choice. The output already has standardized coordinates and
is not standardized again at deployment.

The canonical mixed-task encoder uses L2-normalized frozen BGE-large-en-v1.5 features with a maximum input length of 512 tokens \citep{bge_modelcard}. Long questions can therefore lose relevant content before the trainable head; Proposition~\ref{prop:routing} does not eliminate that information loss. The head is $1024\to512\to512\to K$, with ReLU and dropout $.1$ after each hidden layer. Adam uses learning rate $.001$, weight decay $.0001$, and batch size 32. The construction pool has $N=600$ queries. A stratified outer split reserves 15\% (90 queries) for encoder validation, leaving $N_{\mathrm{enc}}=510$ for encoder fitting. Within those 510 rows, another 15\% selects the epoch count by combined loss, with maximum 200 epochs and patience 25. The final head is reinitialized and trained on the full outer-training fold for that count. Initialization and outer splitting use seed 13; the inner split uses seed 14. Normalization is fitted on the outer-training rows; it is not refitted to exclude the inner-selection rows. The outer evaluation remains excluded from normalization and epoch selection.

The choice of frozen features is motivated by limited fingerprint supervision and the ability to reuse a cached text representation. The width and loss weights are fixed capacity and optimization choices; no experiment establishes them as optimal. Predicting continuous coordinates instead of cluster labels retains information for alternative partitions and operator profiles, at the cost of learning variations that may be irrelevant to the current expert bank. A direct cluster classifier is therefore an informative ablation.

Cluster construction uses Euclidean $k$-means with $J=6$, initialization seed 13, and ten restarts. The centroids use the encoder checkpoint's normalization. Clustering uses the construction pool, including rows reserved for encoder validation; those rows remain held out for the \emph{encoder} fit but not for the bank construction. The separate test pool evaluates the composed system. Six clusters allow repeated use of one offline expert while retaining tens of construction examples in the smaller clusters. Neither cluster count nor squared-distance clustering is justified by a utility-optimality guarantee. Historical eight-cluster diagnostics are not the deployed six-cluster configuration.

Let $C=(c_1,\ldots,c_J)$ be the fitted centroids, with
$c_j\in\mathbb R^K$. For any vector $v$ in the standardized embedding space,
the assignment rule and construction members are
\begin{equation}
r_C(v)=\argmin_{1\le j\le J}\|v-c_j\|_2^2,\qquad
\mathcal T_j=\{q\in\mathcal T:r_C(\bar z(q))=j\}.
\label{eq:clustering}
\end{equation}
The norm is Euclidean, and ties use a fixed rule. A clustering centroid $c_j$
summarizes its full cluster; the generation embedding $\rho_j$ averages the
search-validation subset defined below and need not equal $c_j$.
The canonical pipeline uses one fitted clustering for construction and
assignment. The recorded embedding-guided experiment instead translates
between two stored clusterings using a training-fitted label mapping.
Appendix~\ref{app:guided-construction} specifies that mapping and its agreement
rates; it is part of the evaluated implementation rather than an exact
identity between the two geometries.

\subsection{Construction backends and their experimental roles}
Section~\ref{sec:generation} separates the information supplied to construction from the optimizer that consumes it. The main mixed-task baseline comparison uses \emph{subset-conditioned construction}: six experts are searched with AFlow on the six affinity-defined training subsets. Its baseline operator vocabulary contains Custom and ScEnsemble. Each search permits five optimization rounds after its initial candidate. The global comparator $D$ uses the same search rule and round limit on the undivided training pool. The validation cap is 600 in both configurations, so the cluster sizes sum to the same nominal 600 candidate--query validations per round as $D$. Six searches nevertheless require more proposal calls, and different workflows incur different execution costs; this is not equality of aggregate construction tokens or dollars. Semantic and random partitions change the examples assigned to each search while preserving this construction interface.

The SWE-bench experiment instantiates the same cluster-to-bank idea by selecting among policy prompts within a fixed localize--repair executor. Its global $D$ is the training-selected policy, whereas the separately reported AFlow adaptation searches policy instructions. Finally, the embedding-guided search experiment below supplies the search procedure with affinity-derived seeds, an operator prior, and failure-conditioned feedback. It directly exercises the use of cluster embeddings described in Section~\ref{sec:generation}. Results from the subset-conditioned and policy-selection settings establish the utility of organizing the bank through affinities; the embedding-guided search experiment evaluates the additional construction mechanism. The backend and policy variants are component studies of the common representation, encoder, and generation framework.

\subsection{Expert language for embedding-guided search}
An expert is a dictionary containing an unordered operator set, an integer sample count between one and five, and an instruction string. Optional decomposition adds a short plan to the question. The interpreter always produces the specified number of CoT candidates; when either persona flag is present, two more candidates are added rather than substituted. If multiple candidates survive and sampling-ensemble or persona selection is enabled, a separate call chooses an answer. An optional revision then checks and rewrites it. All scored answer stages receive the same per-domain answer-format instruction. Empty final outputs and scoring failures receive zero; truncated outputs are retained and scored.

If the configured stages complete, their nominal call count is
\begin{equation}
c(w)=d+n+2p+\mathbf 1\{n+2p>1\}\mathbf 1\{e=1\text{ or }p=1\}+r,
\end{equation}
where $d,p,e,r$ indicate decomposition, persona generation, ensemble selection, and revision, and $n$ is the sample count. With fewer surviving candidates, the actual selection stage can differ; empirical costs use recorded calls. When $n>1$ without a selection flag, the interpreter returns the first surviving candidate. Listing both persona aliases does not add another branch. These facts are disclosed to the optimizer and should be used when counting distinct executed workflows.

For each cluster, $V_j\subseteq\mathcal T_j$ is the search-validation set,
$n_j=|V_j|$, and $H_j=\mathcal T_j\setminus V_j$ contains the remaining
construction examples. The recorded optimizer receives statistics involving
$H_j$, so this remainder is not an independent holdout. The cluster embedding
is $\rho_j=n_j^{-1}\sum_{q\in V_j}\bar z(q)$. Search also receives the
cached probe outcomes for individual queries in $\mathcal T_j$: the average
embedding alone cannot provide probe outcomes on the current failures.
The LLM optimizer receives selected texts of queries on which the current
workflow scores below one, mean raw probe scores on the failed and solved
subsets, whole-cluster raw probe averages, the current workflow, and previous
attempts. These raw feedback summaries are distinct from the standardized
coordinates used to rank probes in $\rho_j$.
For a candidate $w$, define its recorded search score as
\begin{equation}
\widehat U_{V_j}(w)=\frac{1}{n_j}\sum_{q\in V_j}\operatorname{score}(w,q),
\end{equation}
where each summand is the task score recorded from one execution of that
candidate. This empirical mean is distinct from the expected cluster utility
$U_j(w)$ in the analysis. A proposed workflow $w'$ replaces the current
workflow $w$ only when
\begin{equation}
\widehat U_{V_j}(w')>\widehat U_{V_j}(w)+\frac{2}{n_j}.
\label{eq:search-acceptance}
\end{equation}

\begin{algorithm}[t]
\caption{Recorded embedding-guided search for one cluster}
\label{alg:generation}
\begin{algorithmic}[1]
\Require Construction members $\mathcal T_j$, cached per-query probe outcomes, operator map, LLM optimizer, proposal budget and resource limits
\State Split $\mathcal T_j$ into search-validation set $V_j$ and remaining examples $H_j$
\State Set $n_j\gets|V_j|$; compute $\rho_j$ from $V_j$ and rank its probe coordinates
\State Choose the highest-ranked runnable probe and map it to seed $v_j$
\State Restrict operators using the bottom-three rule, exempting the mapped seed operator
\State $w\gets v_j$; evaluate $\widehat U_{V_j}(w)$ and retain per-query outcomes
\For{each remaining proposal round (two in the recorded run)}
  \State Stop if a configured cost or time guard has been reached
  \State Identify failures of $w$ on $V_j$ and their probe outcomes
  \State Form feedback with failure texts, mean raw probe scores for failed/solved/all cluster members, $w$, and prior attempts
  \State Ask the LLM optimizer for a specification edit; parse and sanitize it
  \If{invalid or previously evaluated specification}
    \State Record the event and continue
  \EndIf
  \State Evaluate candidate $w'$ on the same $V_j$
  \If{$\widehat U_{V_j}(w')>\widehat U_{V_j}(w)+2/n_j$}
    \State $w\gets w'$; retain its outcomes for the next round
  \EndIf
\EndFor
\State \Return $w$ and the full construction trace
\end{algorithmic}
\end{algorithm}

The local validation cap is 60 queries, with at least 20 reserved locally when possible; small clusters use the remaining available validation members. Construction uses one execution per candidate, while separate final comparisons use repeated executions. Search also stops at its configured cost or wall-time guard. These guards bound resource use rather than optimize the same objective as score acceptance. The two-grid-step margin was introduced to avoid accepting small fluctuations; with heterogeneous fractional scores and adaptive proposals, it is not a confidence bound.

The bottom-three restriction ranks within the cluster, excludes an operator only when all mapped probes are low-ranked, and exempts the single operator mapped from the selected seed probe. This handles many-to-one mappings without letting one low-ranked alias veto a well-supported operator. The implementation also has a nonempty-set fallback. The prior reduces available search choices; standalone outcomes do not identify all beneficial interactions, so its exclusion loss may be positive. The three-round budget was a limited construction allocation, not evidence of search convergence.

The optimizer is asked to edit one \emph{axis}: sample count, addition to the operator set, removal from it, or instruction. For recognized edit types, the sanitizer resets fields outside that axis. An addition/removal request may still change several flags on its axis. Unknown edit types are logged and do not receive the same axis constraint, although vocabulary and prior filtering still apply. No unrestricted program synthesis or Monte Carlo tree search occurs in this interpreter.

\paragraph{Selection independence.}
The cluster embedding and seed use $V_j$. However, the recorded feedback's whole-cluster column averages over $\mathcal T_j=V_j\cup H_j$, exposing probe outcomes from the remaining examples $H_j$. Thus $H_j$ is not an independent certificate of the search procedure. The later 300-query test pool is a separate evaluation pool, but controls and analyses added after observing its results remain exploratory. To apply the independent-selection result in Appendix~\ref{app:proofs}, candidates must be frozen before an untouched selection sample is evaluated, or a justified uniform guarantee must cover the actual adaptive procedure. Reusing $V_j$ and simply requiring a positive empirical margin does not meet this condition.

\subsection{Why the design choices remain ablations}
The central objects are the outcome coordinates, their text approximation, and their use in construction. Changing a probe subset, normalization, feature backbone, number of clusters, assignment rule, or search optimizer changes an implementation of those objects. The corresponding intervention should hold the remaining executor, score contract, split, and resource allocation fixed. Behavioral versus semantic or random construction needs the same operator language and proposal budget; seed-only versus searched experts needs the same assignments. Measured A versus text B changes information access, so it diagnoses amortization rather than supplying a deployable competitor. These controls follow directly from the quantity each proposition characterizes.

The design also separates construction cost from serving cost. Probing and expert generation are offline investments; a text-only assignment incurs local encoder computation and the chosen expert's calls. An amortization advantage requires both a stated quality target and lower per-query serving cost than the comparator. Since guided B uses more calls than single-call CoT without an established quality improvement, this experiment does not justify a break-even claim against CoT.

\section{Experimental protocol and additional evidence}
\label{app:experiments}

\subsection{Data, scores, and information access}

The mixed-task evaluation uses the repository's persisted \texttt{mixed\_hard\_xl} pool with seed 13. Training contains 200 AIME, 200 GPQA, 100 CodeContests, and 100 LiveCodeBench queries; testing contains 100, 100, 50, and 50, respectively. The pool search procedure samples from the available source validation and test material, assigns stable source identifiers, and persists disjoint project training/test memberships. These are project splits, not a claim to reproduce the benchmarks' official evaluation protocols. Affinity measurements use labels or executable tests offline. The encoder sees the query text; a deployed $B$ route does not measure a test-query fingerprint. Outcome-assisted $A$ routing does measure the test-query fingerprint, and its results include information unavailable to deployment.

AIME is scored by the existing mathematical answer checker. GPQA extracts the option following the last \texttt{Answer:} marker, falling back to a letter match if that marker is absent. Every scored September arm receives the appropriate answer-format instruction. For coding, the scorer extracts Python code, executes at most 12 available public stdin/stdout tests with a ten-second limit per test, normalizes output whitespace, and returns the fraction passed. Unsupported functional or interactive tests are skipped. Missing runnable tests score zero. Public-test success need not imply correctness on hidden tests, and the aggregate metric weights each query equally rather than each benchmark equally. The implementation's scoring and loading files are included in the provenance manifest.

\subsection{Historical probe measurements and cross-repeat analysis}

The twelve mixed-task probe coordinates, in stored order, are direct, CoT, CoT self-consistency, tree-of-thought, self-refinement, reflexion, retrieval, ReAct, tool use, debate, role team, and decomposition. These labels name concrete repository implementations; they do not imply exact reproduction of every original paper's system. Both historical model panels have 21,600 unique training cells and 10,800 unique test cells: 600/300 queries, twelve probes, and three repeats. The offline audit verified completeness of every query--probe--repeat tensor and used the recorded scores without rerunning models.

The retrieval probe uses same-benchmark TF--IDF demonstrations with available reference answers and excludes the current query. The current generator explicitly supplies the training corpus when invoked for test fingerprints. The saved sample ledgers retain answer previews rather than complete prompts or retrieved identifiers, so they do not independently certify the retrieval membership of every historical call. The September routed-probe arm selects no retrieval probe. For future releases, complete retrieval provenance should accompany the scorer and model revision, since conditioning on another evaluation query's label would change the information contract.

For query $i$, probe $k$, and held-out repeat $s$, let
\[
 a_{ik}^{(-s)}=\frac{1}{2}\sum_{t\ne s}Y_{ikt},
 \qquad
 \widehat{k}_{i}^{(-s)}\in\arg\max_k a_{ik}^{(-s)}.
\]
The cross-repeat score is $3^{-1}\sum_sY_{i,\widehat{k}_{i}^{(-s)},s}$. Training-priority ties use decreasing mean score across the 600 training queries, with stored probe order as a final deterministic tie rule. The uniform-tie analysis computes the exact mean held-out score across all tied maximizers, rather than drawing an additional random tie. The fixed probe is tool use for GPT-4o-mini and CoT self-consistency for Haiku. Its test score also averages the three repeats. The observed oracle instead takes the maximum of the three-repeat averages and hence reuses execution noise for selection and evaluation.

We freshly computed the cross-repeat confidence intervals from the saved ledgers using 10,000 benchmark-stratified bootstrap resamples, seed 13. Each resample draws queries with replacement within their benchmark and preserves the observed benchmark counts. The three overlapping held-out-repeat folds are averaged per query before resampling; they are not treated as independent observations. These intervals describe variability across this query mixture conditional on the saved training comparator and recorded executions. They do not incorporate retraining, changed model revisions, or uncertainty from choosing the probe panel. In particular, the Haiku training-priority improvement is $0.0343$ with interval $[0.0110,0.0592]$, while the uniform-tie interval includes zero. This difference makes tie handling part of the scientific result.

\subsection{Encoder and representation ablations}

The expanded-panel study fits an MLP on frozen BGE features, using a 510/90 outer training/validation partition stratified by benchmark. Its checkpoint records nested inner-split epoch selection followed by refitting on outer-training rows, selecting 24 epochs. Normalization statistics are fit on the permitted training partition. The reported MSE averages over coordinates; converting it to the expected squared Euclidean error used in the analysis requires multiplication by the number of coordinates. The study's 90-query validation metrics should not be confused with performance on the separate 300-query routing test set. Text truncation in the frozen backbone is a further information boundary: relevant information outside its retained input cannot be recovered by the MLP.

The paired eight/twelve-probe comparison preserves query membership and the eight shared measurements. Adding tree-of-thought, reflexion, ReAct, and role team increases the number of distinct measured fingerprints from 281 to 329 and effective rank from 2.6698 to 3.0819. On the 90 validation queries, the observed menu oracle increases from $0.4324$ to $0.4503$. This is increased descriptive resolution and observed menu coverage; because maxima are computed on the same measurements, it is not independent evidence of latent workflow sufficiency. Native cosine/MSE scores across different target widths are also not directly paired estimands. The repository contains two neighbor-recall conventions yielding different values; we omit a headline neighbor-recall claim pending reconciliation of tie handling and distance conventions.

\begin{table}[t]
\centering\small
\caption{Historical mixed-task routing ablations. All values within a model column share that model's July evaluation artifact. These are not same-harness comparisons to September guided generation.}
\label{tab:all-historical}
\begin{tabular}{lrr}
\toprule
Variant & GPT-4o-mini & Haiku 4.5\\
\midrule
Global searched workflow ($D$) & .2722 & .5381\\
Twelve-probe measured $A$ & .2742 & .6097\\
Twelve-probe predicted $B$ & .2922 & .5214\\
Eight-probe measured $A$ & .2844 & .6053\\
Eight-probe predicted $B$ & .3072 & .5083\\
Semantic BGE partition & .2939 & .5303\\
Semantic Qwen3 partition & .3050 & .5861\\
Random partition & .2767 & .5228\\
\bottomrule
\end{tabular}
\end{table}

Table~\ref{tab:all-historical} treats panel width and partition choice as ablations. Historical inference uses query-paired confidence intervals and sign-flip tests with Holm adjustment within the stored comparison families. The twelve-probe $B-D$ differences are $0.0200$ ($[-0.0150,0.0558]$) for GPT-4o-mini and $-0.0167$ ($[-0.0661,0.0333]$) for Haiku; both adjusted $p$-values equal one. For a separate execution-attribution diagnostic, we hold the paper-native $B$ routes fixed and replace each searched expert's stored output score by the score of its training-selected raw probe. Searched-minus-probe differences are $-0.0077$ ($[-0.0500,0.0355]$) and $-0.1442$ ($[-0.1931,-0.0950]$), respectively, using freshly computed stratified intervals. Probe scores average three executions, whereas the searched workflow scores follow their stored evaluation protocol. This identifies a possible loss at the execution interface; it is not a controlled test of a particular search algorithm.

\subsection{Guided construction, splits, and executable semantics}
\label{app:guided-construction}

\begin{table}[t]
\centering\small
\caption{Six guided searches. Each has three rounds including initialization, and one execution per search-validation query per evaluated candidate. C3 is the only cluster with an accepted edit. C0/C4 and C2/C5 execute identical final specifications.}
\label{tab:guided-searches}
\begin{tabular}{lrrllr}
\toprule
Cluster & Train & Search/local holdout & Top probe & Final behavior & Calls\\
\midrule
C0 &301&60/241&Self-refine&CoT + refinement&2\\
C1 &83&60/23&Direct&Single CoT&1\\
C2 &42&22/20&CoT-SC&Three samples + selection&4\\
C3 &48&28/20&ReAct&Plan + four samples + selection&6\\
C4 &82&60/22&Self-refine&CoT + refinement&2\\
C5 &44&24/20&Tool use&Three samples + selection&4\\
\bottomrule
\end{tabular}
\end{table}

In C3, the three-sample seed scores $0.2679$ on the 28 search-validation queries. Increasing to four samples yields $0.4107$; adding decomposition yields $0.4821$. These are adaptive search scores, not test estimates. The final increment lies at the $2/28$ acceptance boundary in exact fractions and passes the recorded floating-point comparison.

The search uses an unordered set of JSON flags in a fixed plan--candidate--selection--refinement interpreter. Single candidates use temperature zero; multiple candidates use temperature $0.7$. An optional persona branch adds two candidates. Debate and role-team flags share this branch; direct and CoT flags both retain the always-on CoT instruction. Reflexion maps to self-refinement and tree-of-thought to a sampling ensemble. ReAct and tool use have no executor branch. Thus similarity of probe and operator names does not certify equality of their executed behavior.

Seed profiles use search-validation members rather than local holdout members. However, the optimizer's marginal-value table computes its \texttt{ALL} column over all cluster members, including the local holdout. Furthermore, candidate proposals and acceptance reuse the same search-validation outcomes. Consequently neither the local holdout nor adaptively reused validation establishes the independent-selection assumptions of the conditional analysis. Only C1/C2 have stored local-holdout executions; we do not use them as independent confirmation. The outer 300 test queries do not participate in these searches.

The deployed expert bank was constructed in the \texttt{mixedhardxl12} clustering, while recorded paper routes use a separate \texttt{paper\_mhxl12} clustering. A Hungarian bijection fitted solely to their 600 training assignments maps paper labels $0,1,2,3,4,5$ to expert labels $4,2,1,3,5,0$, with 517/600 agreements. Native versus translated measured test assignments agree on only 80.33\% of queries. This bridge is part of the evaluated implementation and is not equivalent to training and deploying within one fixed cluster geometry.

\subsection{September controlled comparisons and cost accounting}

The guided, direct, and routed-probe arms share the same 300 queries and official OpenAI endpoint. The single-call controls were executed subsequently on the same queries using the guided executor's unmodified single-call implementation. Their comparison is paired by query, not by identical random seeds. Guided and single-call controls use a 3,000-token cap; the pipeline Direct and probe-menu arms use 4,000 tokens and provider-default temperature. All use the scorer described above. Errors or empty final outputs score zero. None of these six arms has a recorded failed call or final-query error, but 40 guided calls were truncated (11 in $A$, 29 in $B$), as were nine calls in each single-call control. Truncated outputs remain scored as returned.

\begin{table}[t]
\centering\small
\caption{Stored September results on 300 queries and three repeats. SD is across the three pass means, not a standard error over queries. Costs are token-based estimates using the rates configured for those runs. Outcome-assisted $A$ serving calls omit the cost of obtaining its test fingerprints.}
\label{tab:guided-full}
\begin{tabular}{lrrrr}
\toprule
Arm & Mean & Pass SD & Calls/query & Total USD\\
\midrule
Direct & .2763 & .0031 & 1.000 & .2701\\
Single bare & .3121 & .0083 & 1.000 & .3159\\
Single CoT & .3295 & .0180 & 1.000 & .4702\\
Probe-menu-$B$ & .2955 & .0133 & 2.924 & 1.0924\\
Guided-$B$ & .3358 & .0097 & 2.113 & .9919\\
Guided-$A$ (diagnostic) & .3413 & .0169 & 2.403 & 1.1613\\
\bottomrule
\end{tabular}
\end{table}

The displayed guided contrasts preserve the stored 10,000-resample percentile query-bootstrap procedure: average each query across its three executions, then resample queries without benchmark stratification. These exploratory intervals are unadjusted for multiplicity. They do not resample construction, encoder fitting, or provider drift. Guided-$B$ minus probe-menu-$B$ is $0.0404$ ($[0.0042,0.0773]$), but the arms differ in executor, prompting, decoding, and token cap; the contrast cannot isolate embedding guidance. Six construction runs total 2,544 calls and an estimated USD~1.2038 including the limited local-holdout arms. Offline probe generation and encoder fitting are additional costs, excluded from serving calls/query. Three separate repeats were verified from per-pass records even though the merged four-arm artifact retains its original two-repeat command arguments.

\subsection{Expert-bank and software-engineering boundary checks}

A separate single-execution $6\times300$ matrix evaluates every generated expert on every mixed-task test query. It yields guided-$A$/$B$ readouts of $0.3397/0.3386$, versus $0.3323$ for the expectation of uniform expert selection. These values are a distinct execution sample from Table~\ref{tab:guided-full}. Cluster-specific superiority is not established: no cluster's own-versus-best-other interval excludes zero. The test-selected best fixed expert scores $0.3567$, but selection on the same test set makes it a diagnostic rather than a valid training-selected baseline. Counterfactual seed-family readouts cover only five of twelve probes and three seed families; they do not replace matched searches with guided versus unguided conditioning.

The historical SWE boundary study uses 500 SWE-bench Verified instances~\citep{swebench2024}, a persisted repository-stratified 300/200 train/test split, and a single-turn Agentless-lite localize--repair executor. Twelve policy texts form the measured panel; ReAct, reflexion, and tool-use labels are single-turn analogues here. The metric is actual harness resolution, not the older file-localization proxy elsewhere in the repository. A per-policy logistic predictor replaces the mixed-task MLP, and small clusters merge from six to four. Predicted routing chooses direct on all 200 tests and scores $0.200$, versus the training-selected global retrieval policy's $0.175$; the paired interval for the difference is $[-0.020,0.070]$. All failures are retained as unresolved. The 6,000-cell grid contains 1,565 errored cells, including 210 infrastructure/no-report residues. This is a transfer stress test with changed executor and predictor, not replication of the guided generator.

\subsection{Baseline definitions and adaptations}
\label{app:baseline-adaptations}

We compare published mechanisms after adapting them to the project's query pools, executors, answer formats, and scorers. The adaptation names in Table~\ref{tab:main-baselines} specify the evaluated systems. They do not stand for the original papers' full configurations or reported benchmark scores. Query-level error handling is fail-closed: unsuccessful execution contributes zero.

\paragraph{Global and cluster-conditioned AFlow.}
The mixed-task bank uses AFlow's workflow search~\citep{aflow2025}, with Custom and ScEnsemble operators, five optimization rounds after initialization, parent sample size four, one validation execution per query, and the executor itself as optimizer. The published default has a longer search and a stronger optimizer. All six cluster partitions together contain the same 600 training queries evaluated by the undivided $D$ search, matching nominal validation-query allocation per round. Six searches nevertheless require six proposal streams rather than one and can incur different token costs; aggregate offline budgets are not equal. Semantic BGE, Qwen3, and random partitions use this same search protocol. The SWE bank selects among twelve policy texts on each cluster; its $D$ is the global training-selected retrieval policy. The distinct SWE AFlow baseline searches a policy on the shared, repository-stratified 150-instance training subset, retaining AFlow's parent-selection rule and one edit per round. Its five-round winner resolves 38/200 test instances.

\paragraph{ADAS.}
Our ADAS adaptation~\citep{adas2025} retains archive-conditioned proposal, reflection/debugging, and bootstrap fitness selection. Its meta-model is GPT-4o-mini or Haiku, matching the executor rather than the original GPT-4-class meta-agent. On mixed tasks it searches agent code; on SWE it searches policy text for the fixed pipeline, from direct/CoT/ReAct-policy seeds over five generations. SWE's winner is the CoT seed (40/200 test resolutions). For Haiku, CPU contention affected the original coding fitness evaluation. We therefore use the clean re-score of all eight archived candidates on the full training pool, selecting Self-Refine by clean training fitness before evaluating on the 300 test queries. Its $0.2489$ replaces the obsolete contended score of $0.4075$.

\paragraph{ScoreFlow.}
The adaptation implements SF-0 candidate generation and an in-context preference-optimization (ICPO) step~\citep{scoreflow2025}; it does not perform Score-DPO weight updates. GPT-4o-mini replaces the fine-tuned workflow generator. Mixed-task inference generates four candidates and selects without test-score feedback; its test utility is $0.2756$. A separate training-selected majority-workflow control gives $0.2758$. On SWE, twelve policy candidates are scored on the common 150 training instances. Preference pairs use the published cubic score-gap weighting; one ICPO refinement is compared against the best SF-0 candidate. SF-0 scores $0.1733$ on training versus ICPO's $0.1533$, and its selected policy resolves 42/200 test instances. The Haiku ScoreFlow configuration was not run.

\paragraph{FlowReasoner.}
Our search-only FlowReasoner adaptation~\citep{flowreasoner2025} omits supervised/reinforcement learning of the meta-reasoner and uses the executor model for proposals. Mixed-task runs search separately for each evaluated query, using its scorer feedback to choose a workflow before three fresh executions of that same query. They are therefore outcome-assisted diagnostics, marked $\dagger$, rather than text-only deployment results. The operator space includes Programmer. GPT-4o-mini completes 297/300 queries; its saved completed-only mean is $0.3601$. Table~\ref{tab:main-baselines} uses $0.3565$ after assigning zero to the three missing queries, retaining the shared 300-query denominator. Haiku completes all 300 and scores $0.7465$. Logged costs average USD~0.00660 and 0.21060 per completed query, with mean wall time 284.2 and 268.7 seconds. Missing-query costs are not fully reconstructed. On SWE, the adaptation instead searches one pool-level policy on the disjoint 150-instance training subset; the stock-pipeline seed wins and resolves 41/200 tests. That SWE result does not use test-instance scorer feedback during selection.

\paragraph{Policy routing on SWE.}
MasRouter~\citep{masrouter2025} instantiates the cost-aware routing objective over the twelve fixed Agentless policies, using its sampled-selection protocol, five training epochs, and cost rate 200. It resolves 34/200 instances at USD~0.00227 per instance. FrugalGPT~\citep{frugalgpt2024} trains issue-text scorers and selects a cascade on training data. At budget USD~0.004041 it selects direct $\rightarrow$ debate $\rightarrow$ retrieval; test utility is $0.2050$ at USD~0.00240. Only 21/200 queries escalate beyond direct. Both methods have access to the same stored policy outcomes for training; per-policy test costs are similar, restricting the available cost-quality trade-off. The higher-budget FrugalGPT configuration selects the identical cascade.

\subsection{The full probe panel and repeated-execution evidence}

\begin{table}[t]
\centering\small
\caption{Test means of the twelve probe coordinates. Mixed-task entries average three executions. SWE entries are single-turn policy analogues under the Agentless pipeline, not implementations of their multi-turn namesakes. $\ddagger$ marks mixed-task probes using execution feedback, an information advantage over the original Custom/ScEnsemble search space.}
\label{tab:full-probes}
\begin{tabular}{lrrr}
\toprule
Probe & Mixed / GPT-4o-mini & Mixed / Haiku & SWE\\
\midrule
Direct & .2872 & .5694 & .200\\
CoT & .2852 & .5953 & .205\\
Self-consistency & .3260 & .6894 & .190\\
Tree-of-thought & .2911 & .5919 & .200\\
Self-refinement & .2993 & .6814 & .145\\
Reflexion$^{\ddagger}$ & .3340 & .6656 & .125\\
Retrieval & .2919 & .6184 & .175\\
ReAct$^{\ddagger}$ & .2552 & .5137 & .165\\
Tool use$^{\ddagger}$ & .3481 & .6100 & .150\\
Debate & .3069 & .6814 & .180\\
Role team & .3091 & .6483 & .210\\
Decomposition & .3254 & .6343 & .055\\
\bottomrule
\end{tabular}
\end{table}

\begin{table}[t]
\centering\small
\caption{Repeated-execution evidence for the affinity panel on 300 test queries. The observed oracle selects and evaluates using the same three-repeat mean. Cross-repeat selection evaluates a held-out execution; both selectors require test-query outcomes. Uniform ties average over tied maximizing probes.}
\label{tab:repeatability}
\begin{tabular}{lrrrr}
\toprule
Model & Fixed & Observed oracle & Cross-repeat & Uniform ties\\
\midrule
GPT-4o-mini & .3481 & .5580 & .4059 & .3862\\
Haiku 4.5 & .6894 & .8154 & .7237 & .6909\\
\bottomrule
\end{tabular}
\end{table}

Probe rankings change with the executor and pipeline: the training-selected GPT-4o-mini mixed-task probe is tool use, Haiku selects self-consistency, and SWE selects retrieval. The mixed/SWE training-grid rank correlation is Spearman $\rho=-0.449$ ($p=0.143$, twelve coordinates); excluding decomposition yields $\rho\approx-0.530$. Decomposition frequently conflicts with SWE localization formatting, so this comparison combines reasoning, format, and executor effects. We find no evidence supporting direct transfer of one regime's probe ranking to another.

\subsection{Operator and panel ablations}

\begin{table}[t]
\centering\small
\caption{GPT-4o-mini operator ablation: 31 fresh searches add Programmer while preserving clusterings, routing, and search caps. Six BGE executions remain failed in the saved artifact and are counted as zero; this row is .2861 on the full denominator rather than a completed-only mean.}
\label{tab:programmer}
\begin{tabular}{lrrr}
\toprule
Variant & Custom + ScEnsemble & + Programmer & Difference\\
\midrule
Twelve-probe $B$ & .2922 & .3153 & +.0231\\
Twelve-probe $A$ & .2742 & .2922 & +.0181\\
Eight-probe $B$ & .3072 & .2742 & -.0331\\
Eight-probe $A$ & .2844 & .2814 & -.0031\\
Semantic Qwen3 & .3050 & .3064 & +.0014\\
Semantic BGE & .2939 & .2861 & -.0078\\
Random & .2767 & .3042 & +.0275\\
Global $D$ & .2722 & .2783 & +.0061\\
\bottomrule
\end{tabular}
\end{table}

Adding executable code is motivated by the gap between the strongest measured coding probes and the original expert operator set. Twelve-probe $B$ reaches $0.3153$, but random partitions also improve and eight-probe $B$ decreases. Because each bank is searched anew, the experiment changes both operator availability and realized search trajectories; it cannot isolate an operator effect from search variation. Table~\ref{tab:all-historical} reports the independent eight/twelve-panel comparisons. The higher-dimensional panel improves descriptive resolution, while downstream utility depends jointly on measurement, assignment, and the learned expert bank.

\subsection{Label-free outcomes and additional diagnostic controls}

A post-hoc alternative estimates affinities from within-probe agreement and cross-probe answer consensus, with calibration fitted only on training queries. The saved mixed-task ledger covers only the experts actually executed by $A$ or $B$, so alternative routes can be scored only where they select one of those recorded experts. On GPT-4o-mini's eight-probe agreement variant, coverage is 201/300 and the covered-subset gain over the same subset of $D$ is $0.0402$ (95\% CI $[-0.005,0.086]$, exploratory $p=0.090$). Haiku's twelve-probe agreement variant covers 279/300 and changes utility by $-0.0143$ ($[-0.068,0.039]$), while measured affinities retain a $0.0806$ advantage on that subset. Full-pool lower/upper bounds for uncovered outcomes are $[0.199,0.529]$ and $[0.498,0.568]$, respectively. These are coverage-restricted diagnostic results, not full-test estimates.

The stored answer previews contain only the first 1,200 characters; approximately 82\% of Haiku generations are clipped. The analysis therefore evaluates the available preview features, not what complete generations could support. A related SWE label-free analysis uses the retired ReAct regime and yields $0.085$ versus $D$'s $0.090$; it is not an Agentless result. The retired ReAct direct score is $0.088$ on its clean eight-policy grid and the FlowReasoner transfer attempt scores $0.060$. We exclude the compromised twelve-policy ReAct scoring-churn results from the principal comparisons.

For the guided bank, evaluating all six experts on all 300 queries produces expert means $0.3197,0.3236,0.3567,0.3403,0.3233,0.3303$. The observed per-query maximum is $0.4961$, although 194 queries give identical scores for all six executions. These maxima select on their own evaluation outcomes and hence quantify observed opportunity rather than an achievable text-only policy. The single-call ladder (Direct $0.2763$, bare $0.3121$, CoT $0.3295$, guided-$A$ $0.3413$) further separates prompt/decoder changes from added workflow computation. Differences between arms are interpreted as complete-configuration contrasts, not a causal decomposition of the embedding's contribution.

\subsection{Heterogeneous-model routing and cost-quality comparisons}
\label{app:router-cost}

\begin{table}[t]
\centering\small
\caption{Separate four-model-pool study on the 300 mixed-task test queries. Costs are logged token usage at the configured prices. The local quantized Llama has zero API charge; GPU cost is excluded. These methods choose among different backbones and are not executor-matched comparisons to Table~\ref{tab:main-baselines}.}
\label{tab:model-routing}
\begin{tabular}{lrr}
\toprule
Method / operating point & Utility & USD/query\\
\midrule
Single GPT-4o-mini & .2969 & .000358\\
Single Haiku 4.5 & .5994 & .005353\\
Single Gemini Flash & .5533 & .006978\\
Single Llama-3.1-8B (AWQ-INT4) & .1886 & .000000\\
MasRouter, default cost rate & .2664 & .002012\\
FrugalGPT, low budget & .3969 & .001636\\
FrugalGPT, medium budget & .4839 & .003213\\
FrugalGPT, high budget & .6228 & .006818\\
Observed per-query pool oracle & .769 & .00256\\
\bottomrule
\end{tabular}
\end{table}

The pool contains GPT-4o-mini, Haiku 4.5, Gemini Flash, and locally served Llama-3.1-8B AWQ-INT4. These replace unavailable original-paper model versions, and DeepSeek is omitted. MasRouter uses its cost-aware policy-gradient objective, learning rate 0.01, batch size 16, five epochs, cost rate 200, and seed 1234. Its selections favor cheaper backbones (142/300 Llama slots); all 18 test errors are scored zero. FrugalGPT replaces exact-match scoring with our mixed-task utility and adapts token limits/formatting to these tasks. Its low/medium budgets select GPT-4o-mini $\rightarrow$ Llama $\rightarrow$ Haiku; the highest budget selects Haiku $\rightarrow$ Llama $\rightarrow$ Gemini. That cascade attains $0.6228$, above its single-Haiku reference $0.5994$, at a higher cost. The pool oracle chooses the cheapest model attaining each query's observed maximum and is not executable without outcomes.

\paragraph{Measurement, construction, and serving costs.}
The full Haiku probe program (eight/twelve probes, training/test) records 115,330 calls and USD~603.90. The SWE 6,000-cell grid records approximately USD~14.02, with 2.25 calls and USD~0.0023 per instance--policy pair. Historical GPT-4o-mini probe/search tokens were not comprehensively metered. Under the Haiku program's average sample cost, a twelve-probe, three-repeat test fingerprint costs approximately USD~0.4026 before executing its selected expert; the SWE twelve-policy fingerprint costs approximately USD~0.0276. The corresponding GPT-4o-mini estimate is roughly USD~0.028, derived from the Haiku calls-per-sample ratio and the separately metered direct-call price, rather than measured on those grids. This cost belongs to diagnostic $A$ only; $B$ uses one frozen encoder forward pass and no routing-time LLM calls. Guided construction's separately metered USD~1.2038 excludes the earlier probe program. Table~\ref{tab:guided-full} gives its serving costs. These distinctions make offline investment, outcome-assisted evaluation, and deployable serving directly interpretable.

\subsection{Statistical families and scope of baseline comparisons}

For each historical mixed-task backbone, the stored Holm families contain seven comparisons against $D$, seven between routing variants, and five against Direct/train-selected probes. The seven $D$ comparisons are twelve/eight-probe $A/B$, BGE, Qwen3, and random. The seven routing comparisons are eight-probe $B$ versus Qwen3, BGE, twelve-probe $B$, eight-probe $A$, and random; Qwen3 versus BGE; and twelve-probe $B$ versus twelve-probe $A$. The five probe comparisons are twelve/eight-probe $B$ versus Direct and the training-selected probe, plus $D$ versus Direct. SWE has separate four-comparison $A$ and $B$ families against $D$, Direct, BGE, and random. The prespecified primary comparison in each regime is twelve-probe $B-D$; none is significant. We retain this definition while presenting the representation analyses as additional findings. The SWE $A-D$ result also has exact McNemar $p=0.001312$ (16 corrected/two broken). Main-table ordering against ADAS, ScoreFlow, FlowReasoner, MasRouter, and FrugalGPT reports point estimates; it does not imply unreported significant pairwise differences. Guided contrasts and new diagnostic analyses remain exploratory and unadjusted.

\subsection{Reproducibility and outstanding attribution}

The accompanying evidence manifests record source paths, SHA-256 hashes, query memberships, per-query cross-repeat scores, recomputed means, and statistical conventions. The original evidence audit and expanded baseline audit perform no API calls and leave experimental artifacts unchanged; the latter additionally checks all reported adapted-search means against their saved per-query records and records corrected failure denominators. Historical panel size, semantic/random partitions, and decoder/executor controls are reported as ablations. A matched full guided-versus-unguided construction experiment, independent final candidate selection, capability-complete probe-to-operator mapping, and held-out margin-stratified encoder analysis are not present in the saved evidence. We therefore report no numerical outcomes for these missing interventions and do not infer their success from the conditional bounds.

\end{document}